\documentclass{article} %

\usepackage[preprint,nonatbib]{neurips_2024}

\def\colorful{0}
\ifnum\colorful=1
\newcommand{\new}[1]{{\color{red} #1}}

\newcommand{\inote}[1]{\footnote{{\bf [Ilias: {#1}\bf ]}}}
\newcommand{\snote}[1]{\footnote{{\bf [Sushrut: {#1}\bf ] }}}

\else
\newcommand{\new}[1]{#1}

\newcommand{\inote}[1]{}
\newcommand{\anote}[1]{}
\newcommand{\tnote}[1]{}
\newcommand{\snote}[1]{}
\newcommand{\lnote}[1]{}

\fi

\usepackage{todonotes}

\author{Shuyao Li\thanks{Equal contribution.}\\
  University of Wisconsin-Madison\\
  \texttt{shuyao.li@wisc.edu}
  \And
  Sushrut Karmalkar\footnotemark[1]\\ %
  University of Wisconsin-Madison\\
  \texttt{skarmalkar@wisc.edu} \\
  \And
  Ilias Diakonikolas \\
  University of Wisconsin-Madison \\
  \texttt{ilias@cs.wisc.edu}
  \And
  Jelena Diakonikolas\\
  University of Wisconsin-Madison \\
  \texttt{jelena@cs.wisc.edu}
}
\usepackage[ruled, vlined, linesnumbered]{algorithm2e} %
\usepackage[utf8]{inputenc} %
\usepackage[T1]{fontenc}    %
\usepackage{hyperref}       %
\usepackage{url}            %
\usepackage{booktabs}       %
\usepackage{amsfonts}       %
\usepackage{nicefrac}       %
\usepackage{microtype}      %
\usepackage{xcolor}         %
\usepackage[
style=alphabetic,
maxbibnames=99,
giveninits=true,
sortcites=true,
sorting=ynt]{biblatex}

\usepackage{minitoc}
\usepackage{float}
\usepackage{times}
\usepackage{amssymb, graphicx}

\usepackage{amsthm}
\usepackage{thm-restate}
\usepackage{mathtools}
\usepackage{cleveref}

\newcounter{thm}
\numberwithin{thm}{section}

\theoremstyle{plain}
\newtheorem{theorem}[thm]{Theorem}
\newtheorem{corollary}[thm]{Corollary}
\newtheorem{lemma}[thm]{Lemma}
\newtheorem{proposition}[thm]{Proposition}
\newtheorem{claim}[thm]{Claim}
\newtheorem{assumption}[thm]{Assumption}
\newtheorem{fact}[thm]{Fact}

\theoremstyle{definition}
\newtheorem{definition}[thm]{Definition}
\newtheorem{problem}[thm]{Problem}

\theoremstyle{remark} %
\newtheorem{remark}[thm]{Remark}

\crefname{fact}{fact}{facts}
\crefname{claim}{claim}{claims}
\crefname{assumption}{assumption}{assumptions}
\crefname{problem}{problem}{problems}
\crefname{theorem}{theorem}{theorems}
\crefname{corollary}{corollary}{corollaries}
\crefname{lemma}{lemma}{lemmas}
\crefname{proposition}{proposition}{propositions}
\crefname{definition}{definition}{definitions}
\crefname{remark}{remark}{remarks}
\crefname{note}{note}{notes}
\crefname{example}{example}{examples}

\usepackage{times}

\usepackage{bm}

\usepackage{dsfont}

\usepackage{color}

\usepackage{thm-restate}

\newcommand{\grad}{\nabla}

\def\1{\bm{1}}

\def\eps{{\epsilon}}

\def\vzero{{\bm{0}}}

\def\va{{\bm{a}}}

\def\ve{{\bm{e}}}

\def\vg{{\bm{g}}}

\def\vp{{\bm{p}}}

\def\vu{{\bm{u}}}
\def\vv{{\bm{v}}}
\def\vw{{\bm{w}}}
\def\vx{{\bm{x}}}
\def\vy{{\bm{y}}}
\def\vz{{\bm{z}}}
\usepackage{mathrsfs}
\usepackage[cal=boondoxo]{mathalfa}
\def\ep{\widehat{\mathcal{p}}}
\def\eq{\widehat{\mathcal{q}}}

\def\pp{{\mathcal{p}}}
\def\pq{{\mathcal{q}}}

\def\vxh{\widehat{\bm{x}}}

\def\mI{{\bm{I}}}

\def\cB{{\mathcal{B}}}

\def\cE{{\mathcal{E}}}

\def\cG{{\mathcal{G}}}

\def\cP{{\mathcal{P}}}

\def\cR{{\mathcal{R}}}

\newcommand{\E}{\mathbb{E}}

\newcommand{\R}{\mathbb{R}}

\newcommand{\Ind}{\mathbb{I}}

\DeclareMathOperator*{\argmax}{arg\,max}
\DeclareMathOperator*{\argmin}{arg\,min}

\DeclareMathOperator{\sign}{sign}

\newcommand{\innp}[1]{\langle #1 \rangle}

\newcommand{\Gap}{\mathrm{Gap}}

\newcommand{\dd}{\mathrm{d}}

\DeclareMathOperator{\Var}{Var}

\DeclareMathOperator{\op}{op}
\usepackage{xparse}
\NewDocumentCommand{\DeclarePairedDelimiterWithOptionalStar}{ m m m m }{
  \DeclareDocumentCommand{#1}{ s o m }{%
    \IfBooleanTF{##1}{%
      #2#3*{##3}#4%
    }{%
      \IfValueTF{##2}{%
        #2#3[##2]{##3}#4%
      }{%
        #2#3{##3}#4%
      }%
    }%
  }%
}

\DeclarePairedDelimiter\brackets{(}{)}
\DeclarePairedDelimiter\norm{\|}{\|}
\DeclarePairedDelimiter\abs{|}{|}
\DeclarePairedDelimiterX{\divx}[2]{(}{)}{%
  #1\;\delimsize\|\;#2%
}

\DeclarePairedDelimiterWithOptionalStar{\bigO}{O}{\brackets}{}
\DeclarePairedDelimiterWithOptionalStar{\bigtO}{\widetilde{O}}{\brackets}{}
\DeclarePairedDelimiterWithOptionalStar{\bigOmega}{\Omega}{\brackets}{}
\DeclarePairedDelimiterWithOptionalStar{\bigtOmega}{\widetilde{\Omega}}{\brackets}{}
\DeclarePairedDelimiterWithOptionalStar{\opnorm}{}{\norm}{_{\op}}
\DeclarePairedDelimiterWithOptionalStar{\fnorm}{}{\norm}{_{F}}

\makeatletter
\providecommand*{\diff}{
  \@ifnextchar^{\DIfF}{\DIfF^{}}}

\def\DIfF^#1{%
  \mathop{\mathrm{\mathstrut d}}%
  \nolimits^{#1}\gobblespace}
\def\gobblespace{%
  \futurelet\diffarg\opspace}
\def\opspace{%
  \let\DiffSpace\!%
  \ifx\diffarg(%
  \let\DiffSpace\relax
  \else
  \ifx\diffarg[%
  \let\DiffSpace\relax
  \else
  \ifx\diffarg\{%
  \let\DiffSpace\relax
  \fi\fi\fi\DiffSpace}

\title{Learning a Single Neuron Robustly\\ to Distributional Shifts and Adversarial Label Noise}

\DeclareMathOperator{\OPT}{OPT}
\newcommand{\OPThat}{\widehat{\OPT}}
\newcommand{\OPTT}{\OPT_{(2)}}
\newcommand{\OPTThat}{\OPThat_{(2)}}

\def\mathbf{\bm}

\begin{document}

\maketitle

\begin{abstract}
We study the problem of learning a single neuron with respect to the $L_2^2$-loss in the presence of adversarial distribution shifts, where the labels can be arbitrary, and the goal is to find a ``best-fit'' function.
More precisely, given training samples from a reference distribution $\pp_0$, 
the goal is to approximate the vector $\mathbf{w}^*$
which minimizes the squared loss with respect to the worst-case distribution 
that is close in $\chi^2$-divergence to \(\pp_{0}\).
We design a computationally efficient algorithm that recovers a vector $ \hat{\mathbf{w}}$
satisfying 
$\mathbb{E}_{\pp^*} (\sigma(\hat{\mathbf{w}} \cdot \mathbf{x}) - y)^2 \leq C \, \mathbb{E}_{\pp^*} (\sigma(\mathbf{w}^* \cdot \mathbf{x}) - y)^2 + \epsilon$, where $C>1$ is a dimension-independent constant and $(\mathbf{w}^*, \pp^*)$ is the witness attaining the min-max risk
$\min_{\mathbf{w}~:~\|\mathbf{w}\| \leq W} \max_{\pp} \mathbb{E}_{(\mathbf{x}, y) \sim \pp} (\sigma(\mathbf{w} \cdot \mathbf{x}) - y)^2 - \nu \chi^2(\pp, \pp_0)$.
Our algorithm follows a primal-dual framework and is 
designed by directly bounding the risk with respect to the original, nonconvex $L_2^2$ loss. 
From an optimization standpoint, our work opens new avenues for the design of primal-dual algorithms under structured nonconvexity.   
\end{abstract}

\section{Introduction}\label{sec:introduction}

The problem of learning a single neuron from randomly drawn labeled examples 
is a fundamental problem extensively studied in the machine learning literature.  
Given labeled examples $\{(\vx_i, y_i): (\vx_i, y_i) \in \R^{d}\times \R\}_{i=1}^N$ 
drawn from a reference distribution $\pp_0$, 
the goal in this context is to recover a parameter vector $\vw^*_0$ %
that minimizes the squared loss 
$\Lambda_{\sigma, \pp_0}(\vw)$ 
over a ball of radius $W > 0$:
\begin{equation}
    \vw^*_0 := \argmin_{\vw \in \R^d: \|\vw\|_2 \leq W} \Lambda_{\sigma, \pp_0}(\vw); \quad \Lambda_{\sigma, \pp_0}(\vw):= \E_{(\vx, y) \sim \pp_0} (\sigma(\vw \cdot \vx) - y)^2,  
\end{equation} 
where \( \sigma : \R \rightarrow \R \) is a known (typically non-linear)  
non-decreasing activation function (e.g., the ReLU activation \(\sigma(t) = \max(0, t)\)) and we denote by $\OPT_0 = \min_{\vw: \|\vw\|_2 \leq W} \Lambda_{\sigma, \pp_0}(\vw)$ the minimum squared loss.  
In the realizable setting --- where $y = \sigma(\vw_0^* \cdot \vx)$ and thus $\OPT_0 = 0$ ---
this problem is well-understood and 
by now part of the folklore (see, e.g., \cite{kakade2011efficient,kalai2009isotron,yehudai2020learning,soltanolkotabi2017learning}). The results for the realizable setting also naturally extend to zero-mean bounded-variance label noise.

The more realistic %
agnostic (a.k.a.\ adversarial label noise)  model~\cite{Haussler:92, kearns1992toward}  
aims to identify the best-fitting neuron for a reference distribution of the examples,  
without any assumptions on label structure. 
However, it is known that in this setting %
finding a parameter vector with square loss 
$\OPT_0 + \eps$  requires \(d^{\text{poly}(1/\epsilon)}\) time, 
even if the $\vx$-marginal distribution is Gaussian 
\cite{DKZ20-sq-reg, GGK20, DKPZ21-SQ, diakonikolas2023near}.  
Even if we relax our goal to achieve error $O(\OPT_0) + \eps$, 
efficient algorithms only exist under strong distributional assumptions. 
In fact, without such assumptions,  
this problem is NP-hard %
\cite{vsima2002training, MR18}. 
Recent work has also shown that (under cryptographic assumptions) 
no polynomial-time constant-factor improper learner exists 
even for distributions supported on the unit ball~\cite{diakonikolas22a-hardness}. 
Given these intractability results, 
recent work has focused on developing efficient constant-factor approximate 
learners under minimal distributional assumptions 
(see, e.g.,\cite{DGKKS20, frei2020agnostic, DKTZ22, ATV22, 
wang2023robustly-learning, GGKS23, ZWDD2024}).

{This recent progress notwithstanding}, 
prior work primarily focused on the setting where only the {labels} might be corrupted, without considering possible distributional shifts or heterogeneity of the data. 
{Such distributional corruptions} are frequently observed in practice 
and have motivated a long line of research in areas 
such as domain adaptation and (related to it) distributionally robust optimization (DRO); see e.g., \cite{Blanchet2024DistributionallyRO,ben2009robust,namkoong2016stochastic,rahimian2022frameworks} and references therein. Thus, the main question motivating our work is: 
\begin{center}
\emph{How do adversarial changes in the underlying \emph{distribution} impact the learnability of a neuron?}
\end{center}

We study this question within the DRO framework, where 
the goal is to minimize the model's loss on a worst-case distribution from a set of distributions close to the reference distribution.\footnote{\new{We contrast here robustness to perturbed data \emph{distribution} studied within the DRO framework to robustness to perturbed \emph{data examples} referred to as the adversarial robustness in modern deep learning literature (e.g.,~\cite{goodfellow2014explaining}). Our paper is concerned with the former (and not the latter) model of robustness.  %
}} \new{This set of distributions, known as the ambiguity set, models possible distributional shifts of the data.}
{In addition to being interesting on its own merits,} 
the {DRO} framework {arises} in {diverse} contexts, including 
algorithmic fairness~\cite{hashimoto2018fairness} 
and class imbalance~\cite{xu2020class}. Moreover, it has recently 
found a range of applications in reinforcement learning~\cite{Lotidis2023, Yang2023, Wang2023a, Yu2023, Kallus2022, Liu2022}, robotics \cite{Sharma2020}, language modeling \cite{Liu2021}, 
sparse neural network training \cite{Sapkota2023}, and defense against model extraction \cite{Wang2023b}. 

Despite a range of impressive results in the DRO literature %
(see, e.g., recent surveys \cite{rahimian2022frameworks,chen2020distributionally,Blanchet2024DistributionallyRO,kuhn2019wasserstein} and references therein), algorithmic results with rigorous approximation guarantees for the loss have almost exclusively been obtained under fairly strong assumptions about the loss function  involving both convexity and either smoothness or Lipschitzness,  
with linear regression being the prototypical example; see, e.g.,~\cite{blanchet2021statistical,duchi2021learning,chen2018robust}. 
Unfortunately, this vanilla setting does not capture a range of machine learning applications, 
where a typical loss function is nonconvex. In particular, even the simplest ReLU learning problem in the realizable setting (with noise-free labels) is nonconvex. Further, existing DRO approaches for nonconvex loss functions such as \cite{qi2021online,sinha2018certifying} only guarantee convergence to a stationary point, which is insufficient for learning a ReLU neuron even without distributional ambiguity \cite{yehudai2020learning}. 
Motivated by this gap in our understanding, 
in this work we initiate a rigorous algorithmic investigation 
of learning a neuron (arguably the simplest non-convex problem) 
in the DRO setting.  
We hope that this work will stimulate future research in this direction, potentially addressing more complex models in a principled manner. 

Due to space constraints, we defer further discussion of related work to \Cref{sec:related_work}.

\subsection{Problem Setup}\label{sec:problem-setup}

To formally define our setting, 
we recall the definition of \(\chi^2\)-divergence 
between distributions $\pp$ and $\pp'$, given by 
\(
\chi^2(\pp, \pp') := \int \big(\frac{\mathrm{d}\pp}{\mathrm{d}\pp'} - 1\big)^2 \mathrm{d}\pp'.
\) 
We focus on the class of monotone unbounded activations introduced in \cite{DKTZ22}, 
for which we additionally assume convexity. Example activations in this class include the ReLU, leaky ReLU, 
exponential linear unit (ELU), and normalized\footnote{Normalization, which ensures $\sigma(0) = 0$, is without loss of generality, as it corresponds to a simple change of variable: $\hat{\sigma}(t)\leftarrow \sigma(t) - \sigma(0)$ and $\hat{y} \leftarrow y + \sigma(0),$ which does not affect the loss value or its approximation.} SoftPlus. 

\begin{definition}[Unbounded \cite{DKTZ22} + Convex Activation]
  \label{def:activation}
  Let \(\sigma: \mathbb{R} \rightarrow \mathbb{R}\) be a non-decreasing convex function, and let \(\alpha, \beta > 0\). We say \(\sigma\) is %
  \((\alpha, \beta)\)-unbounded if it satisfies the following: ($i$) \(\sigma\) is \(\beta\)-Lipschitz; ($ii$)  $\sigma(t_1) - \sigma(t_2) \geq \alpha(t_1 - t_2)$ for all $t_1 \geq t_2 \geq 0$, and ($iii$) \(\sigma(0) = 0\).
\end{definition}
To formally state the problem, we further define the loss, risk, and optimal value (denoted by $\OPT$). 

\begin{definition}[Loss, Risk, and $\OPT$]
\label{def:loss_risk_opt}
Given a regularization parameter \(\nu\) and a reference distribution \(\pp_{0}\), 
let $\cP = \cP(\pp_0) $ denote the set of all distributions 
that are absolutely continuous with respect to $\pp_0$ and $%
\cB(W) := \{ \vw : \| \vw  \|_2 \leq W\}$. We define the following:
\begin{align*}
    L_\sigma(\vw, \pp; \pp_{0}) &:= \E_{(\vx, y) \sim \pp} (\sigma(\vw \cdot \vx) - y)^{2} - \nu \chi^2(\pp, \pp_{0}) = \Lambda_{\sigma, \pp}(\vw) - \nu \chi^2(\pp, \pp_{0}),\\
  R(\vw; \pp_{0}) & := \max_{\pp \in \cP(\pp_0)} L_\sigma(\vw, \pp; \pp_{0}),\; \pq_{\vw}:= \argmax_{\pp \in \cP(\pp_0)} L_\sigma(\vw, \pp; \pp_{0}),\\  %
  \vw^{*} & := \argmin_{\vw \in \cB(W)} R(\vw; \pp_{0}), \quad \pp^{*} := \pq_{\vw^{*}}, \\
    \OPT &:= \E_{(\vx, y) \sim \pp^{*}} (\sigma(\vw^{*} \cdot \vx) - y)^{2} = \Lambda_{\sigma, \pp^*}(\vw^*).
\end{align*}
 We say that $L_\sigma(\vw, \pp; \pp_{0})$ is the regularized square loss function of a vector \(\vw\) and a distribution \(\pp \in \cP \); and $R(\vw; \pp_{0})$ is the DRO risk of $\vw$ with respect to $\pp_0$. We call $\pp^*$ the target distribution. 
 \end{definition}

 {The minimization of the DRO risk as defined above corresponds to the regularized/penalized DRO formulation studied in prior work; see, e.g.,~\cite{mehta2024primal,Wang2023b,sinha2018certifying}. An alternate formulation would have been to instead optimize over a restricted domain. The two are equivalent because of Lagrangian duality. We show in \Cref{claim:ambiguity-radius} a concrete relation between our regularization parameter $\nu$ and the chi-squared distance between the population distribution $\pp_0$ and the target distribution $\pp^*$. We further require that  $\nu$ is sufficiently large to ensure that the resulting $\chi^2(\pp^*, \pp_0)$ is smaller than an absolute constant, which is in line with the DRO being used for not too large ambiguity sets  \cite{rahimian2022frameworks}.
}

\noindent\textbf{Empirical Version}\quad   
If the reference distribution is the uniform distribution on \(N\) 
labeled examples \((\vx_{i}, y_{i}) \in \R^{d} \times \R\) drawn from $\pp_0$, 
we call it $\ep_0 = \ep_0(N)$, and similarly define $\ep \in \cP(\ep_0)$. 
Note that \(R(\vw^{*}; \ep_{0}) = \max_{\ep \in \cP(\ep_{0})}\E_{(\vx, y) \sim \ep} (\sigma(\vw^{*} \cdot \vx) - y)^{2} - \nu \chi^2(\ep, \ep_{0})\); 
if we let \(\ep^{*}\) denote the distribution that achieves the maximum, 
\(\ep^{*}\) has the same support as \(\ep_{0}\) and can be interpreted as the reweighting 
of the samples that maximizes the regularized loss. 
Formally, our goal is to solve the following learning problem.

\begin{problem}[{Robustly} Learning a Single Neuron Under Distributional Shifts]\label{prob:main}
Given error parameters $\epsilon, \delta \in (0,1)$,  regularization parameter $\nu > 0,$ set radius $W > 0,$ and sample access to labeled examples $(\vx, y)$ drawn i.i.d.\ from an unknown reference distribution $\pp_0,$ output a parameter vector $\hat{\vw}\in \cB(W)$ that is competitive with the DRO risk minimizer $\vw^* = \argmin_{\vw \in \cB(W)} R(\vw; \pp_0)$ in the sense that with probability at least $1- \delta,$ \( \|\hat \vw - \vw^*\|_2^2 \leq C {\OPT} + \eps\) for an absolute constant $C.$
\end{problem}
While the stated goal is expressed in terms of \( \| \hat{\vw} - \vw^* \|_2 \), 
under mild distributional assumptions that we make on the reference and target distributions,
this guarantee implies being competitive with the best-fit function on $\pp^*$ in terms of \new{both the square loss and the risk}, namely 
\new{$\Lambda_{\sigma, \pp^*}(\hat{\vw}) = O(\OPT) + \epsilon$ and $R(\vw, \pp_0) - \min_{\vw \in \cB(W)} R(\hat\vw, \pp_0) \le O(\OPT) + \epsilon$}. %
Further, our algorithm is primal-dual and it outputs a distribution $\ep$ that is close to $\ep^*$ in the chi-squared divergence.

\new{Since the solution to \Cref{prob:main} has an error of \(O(\OPT) + \epsilon\), when we use the term ``convergence'' in our paper, we refer to the following weaker notion: the iterates of our algorithm \emph{converge} to the (set of) solutions such that asymptotically all iterates lie within the set of \(O(\OPT) + \epsilon\) solutions, which are the target solutions, as stated in \Cref{prob:main}.}

\subsection{Main Result}\label{sec:overview}

Our main contribution is the first polynomial {sample and} time algorithm 
for learning a neuron in a distributionally robust setting for a broad class of activations (\Cref{def:activation}) and under mild distributional assumption on the target distribution (\Cref{assump:margin,assump:concentration} in \Cref{sec:distr-assumpt}).  %
\begin{theorem}[Main Theorem --- Informal]\label{thm:main-informal}
    Suppose that the learner has access to \( N =  \tilde{\Omega}(d /\epsilon^2) \) 
    samples drawn from the reference distribution \({\pp_0} \). 
    If all samples are bounded and the distribution \( \pp^* \) 
    satisfies the ``margin-like'' condition 
    and concentration (\Cref{assump:margin,assump:concentration} in \Cref{sec:distr-assumpt}), then after $\bigtO{d \log(1/\epsilon)}$ iterations, 
    each running in sample near-linear time, with high probability \Cref{alg:main} recovers $\hat{\vw}$ such that  $\norm{\hat{\vw} - \vw^*}_2^2 \le C \OPT + \epsilon$, for an absolute constant $C$.
\end{theorem}
We emphasize that \Cref{thm:main-informal} {simultaneously} 
addresses two types of robustness: firstly, robustness concerning labels (\(y\));  
and secondly, robustness due to shifts in the distribution (\(\pp_0\) being perturbed). This result is new even when specialized to any nontrivial activation like ReLU, realizable case (where $\OPT = 0$), and the simplest Gaussian $\vx$-marginal distribution. 
Without distributional robustness, existing approaches, as previously discussed, yield an error of 
\( O(\OPT) + \epsilon \) under certain $\vx$-marginal conditions. We demonstrate that this error rate
can be also achieved with respect to \( \pp^* \) in a distributionally robust context, as long as 
\( \pp^* \) meets the same conditions specified in \cite{wang2023robustly-learning} --- among the mildest in the literature addressing non-distributionally robust agnostic setting.

\subsection{Technical Overview}\label{sec:techniques}

Our technical approach relies on three main components, described below:

\noindent\textbf{Local Error Bounds}\quad
Our work is inspired by optimization-theory local error bounds (``sharpness'') 
obtained for learning a single neuron with monotone unbounded activations under structured distributions  
without considering distributional shift or ambiguity \cite{MBM2018,wang2023robustly-learning}. 
These bounds are crucial as they quantify growth of a loss function outside the set of target solutions, 
essentially acting as a ``signal'' to guide algorithms toward target solutions in our learning problems. Concretely, under distributional assumptions on $\pp^*$ from \cite{wang2023robustly-learning}, the following sharpness property can be established: there is an absolute constant $c_1 > 0$ such that $\forall \vw \in \cB(2\|\vw^*\|_2)$, \begin{equation}\label{eq:sharpness}
  \|\vw - \vw^*\|_2^2 = \Omega(\OPT) \;\Rightarrow\;   
  \Lambda_{\sigma, \pp^*}(\vw) - \Lambda_{\sigma, \pp^*}(\vw^*) 
  \geq c_1 \|\vw - \vw^*\|_2^2. 
\end{equation}  

The local error bounds in \cite{MBM2018,wang2023robustly-learning} assume identical reference and target distributions. 
Introducing distributional ambiguity --- as in our work --- invalidates this assumption, and as a result necessary distributional assumptions for sharpness may not apply to all distributions in the ambiguity set. %
In this work, distributional assumptions are exclusively applied to the target distribution to exploit the sharpness property proved in \cite{wang2023robustly-learning}. We also assume that the 
sample covariates from the reference distribution are {polynomially}
bounded; this assumption, which is without loss of generality, impacts only the sample and computational complexities and is satisfied by {standard distributions}.

\noindent\textbf{Primal-Dual Algorithm}\quad
Our algorithm is a principled, primal-dual algorithm leveraging the sharpness property on the target distribution, 
the structure of the square loss, and properties of chi-squared divergence. 
We control a ``gap-like'' function of the iterates, \(\Gap(\widehat \vw, \ep; \ep_0) := L_\sigma(\widehat \vw, \ep^*;\ep_0) - L_\sigma(\vw^*, \ep; \ep_0)\). 
The idea of approximating a gap and showing it reduces at a rate \(1/A_k\), 
where \(A_k\) is a monotonically increasing function of \(k\), comes from \cite{Diakonikolas2017TheAD} 
and has been extended to primal-dual methods, including DRO settings, in \cite{song2021variance,diakonikolas2022fast,mehta2024primal,song2022coordinate}.

Unlike past work \cite{song2021variance,diakonikolas2022fast,mehta2024primal,song2022coordinate}, 
our primal problem is nonconvex, even for ReLU activations without distributional ambiguity. Unfortunately, the previously mentioned results relying on convexity do not apply in our setting. Additionally, sharpness --- which appears crucial to approximating the target loss ---  is a \emph{local} property, applying only to \(\vw\)
such that \(\|\vw\|_2 \leq 2\|\vw^*\|_2\), where \(\|\vw^*\|_2\) is unknown.
This condition is trivially met at initialization, but proving it holds for all iterates requires convergence.
We address this issue via an inductive argument, effectively coupling convergence analysis with localization of the iterates.

Additionally, standard primal-dual methods \cite{chambolle2011first,chambolle2018stochastic,alacaoglu2022complexity,song2021variance,song2022coordinate}
rely on bilinear coupling between primal and dual variables in \(L_\sigma(\vw, \ep;\ep_0)\).
In our case, \(L_\sigma(\vw, \ep; \ep_0)\) is \emph{nonlinear} and \emph{nonconvex} in the first argument.
Recent work \cite{mehta2024primal} handled nonlinearity  by linearizing the function using convexity of the loss, which makes the function bounded below by its linear approximation at any point. However, this approach cannot be applied to our problem as the loss is nonconvex. %
Instead, we control the chi-squared divergence between 
the target distribution
and the algorithm dual iterates to bound $L_\sigma(\vw, \ep^*; \ep_0)$ from below, using a key structural result that we establish in \Cref{lemma:main-auxiliary-main-body}. 
The challenges involved in proving this structural result 
require us to rely on chi-squared regularization 
and convex activation \(\sigma\). 
Generalizing our result to all monotone unbounded activations 
and other strongly convex divergences like KL 
would need a similar structural lemma under these broader assumptions.
An interesting aspect of our analysis is that we do not rely on a convex surrogate for our problem. 
Instead, we constructively bound a quantity related to the DRO risk of the original square loss, 
justifying our algorithmic choices directly from the analysis.
Although we do not consider convex surrogates, 
the vector field \(\vv(\vw; \vx, y)\), scaled by \(2\beta\), 
corresponds to the gradient of the convex surrogate loss 
\(\int_0^{\vw \cdot \vx}(\sigma(t) - y) \, \mathrm{d}t\), which has been used in prior literature on learning a single neuron 
under similar settings without distributional ambiguity \cite{kakade2011efficient,DGKKS20,wang2023robustly-learning}.
In our analysis, the vector field \(\vv(\vw; \vx, y)\) is naturally motivated 
by the argument in the proof of \Cref{lemma:main-auxiliary-main-body}.

\noindent\textbf{``Concentration'' of the Target Distribution} \quad
{
To prove that our primal-dual algorithm converges, we need to prove both an upper bound and a lower bound for $\Gap(\widehat \vw, \ep; \ep_0)$. 
The lower bound relies on sharpness; however, we need it to hold for the \emph{empirical target distribution} ($\ep^*$). 
This requires us to translate distributional assumptions and/or their implications from $\pp^*$ to $\ep^*$.
Unfortunately, $\ep^*$ is not the uniform distribution over samples drawn from $\pp^*$. Rather, it is the maximizing distribution in the empirical DRO risk, defined w.r.t.\ \(\ep_0\). 
This means that prior uniform convergence results do not apply.
Additionally, minimax risk rates from prior statistical results, such as those in~\cite{duchi2021learning}, relate $R(\vw; \ep_0)$ and $R(\vw; \pp_0)$.
However, they do not help in our algorithmic analysis since they do not guarantee that the sharpness holds for $\ep^*$.

To address these challenges, we prove (in \Cref{cor:key_fonc-no-max}) that as long as $\nu$ is sufficiently large, there is a simple closed-form expression for \(\ep^*\) as a function of \(\ep_0\) and an analogous relationship holds between \(\pp^*\) and \(\pp_0\). This allows us to leverage the fact that expectations of bounded functions with respect to \(\ep_0\) closely approximate those with respect to \(\pp_0\) to show that expectations with respect to \(\ep^*\) and \(\pp^*\) are similarly close.
This result then implies that the sharpness also holds for $\ep^*$ (\Cref{lemma:sharpness-empirical-appendix}).
Full details are provided in \Cref{app:concentration_pp_star}.}

\section{Preliminaries}\label{sec:prelims}

In this section, we introduce the necessary notation and state basic facts used in our analysis.

\noindent\textbf{Notation}\quad%
Given a positive integer \(N\), \([N]\) denotes the set \(\{1, 2, \dots, N\}\).
Given a set \(\cE\), \(\cE^{c}\) denotes the complement of \(\cE\) when the universe is clear from the context.
We use \(\Ind_{\cE}\) to denote the characteristic function of a set $\cE$: \(\Ind_{\cE}(x) = 1\) if \(x \in \cE\) and \(\Ind_{\cE}(x) = 0\) otherwise. 
For vectors \(\vx\) and \(\vxh\) from the \(d\)-dimensional Euclidean space \(\mathbb{R}^{d}\), we use \(\langle \vx, \vxh \rangle\) and \(\vx \cdot \vxh\) to denote the standard inner product, while \(\|\cdot\|_{2} = \sqrt{\langle \cdot, \cdot \rangle}\) denotes the \(\ell_2\) norm.
We use \((\vx^{(1)}, \vx^{(2)}, \dots, \vx^{(d)})\) to denote the entries of \(\vx \in \mathbb{R}^d\).
We write \(\vx \le \vxh\) to indicate \(\vx^{(j)} \le \vxh^{(j)}\) for all coordinates \(j\).
For \(r > 0\), \(\cB(r) := \{\vx: \|\vx\|_{2}\le r\}\) denotes the centered ball of radius $r$.
We use \(\Delta_{N}\) to denote the probability simplex: \(\Delta_n := \{\vx \in \mathbb{R}^{N}: \sum_{j=1}^{N}\vx^{(j)} = 1, \forall j \in [N] : \vx^{(j)} \ge 0 \}\).
We denote by \(\mI_{d}\) the identity matrix of size \(d \times d\).
We write \( A \succeq B \) to indicate that  \(\vx^\top(A-B)\vx \geq 0\) for all $\vx \in \R^d$. 
For two functions \(f\) and \(g\), we say \(f = \tilde{O}(g)\) if \(f = O(g \log^k(g))\) for some constant \(k\), and similarly define \(\tilde{\Omega}\). We use notation $\Tilde{O}_c(\cdot)$ and $\Tilde{\Omega}_c(\cdot)$ to hide polynomial factors in (typically absolute constant) parameters $c.$ 
For two distributions \(\pp\) and \(\pp'\), we use \(\pp \ll \pp'\) to denote that \(\pp\) is absolutely continuous with respect to \(\pp'\), i.e., for all measurable sets \(A\), \(\pp'(A) = 0\) implies \(\pp(A) = 0\).
Typically, \(\hat{p}\) and \(\hat{q}\) are empirical distributions, and \(\hat{p} \ll \hat{q}\) is equivalent to the condition that the support of \(\hat{p}\) is a subset of the support of \(\hat{q}\).
For \(\pp \ll \pp'\), we use \(\frac{\mathrm{d}\pp}{\mathrm{d}\pp'}\) to denote their Radon–Nikodym derivative, which is the quotient of probability mass functions for discrete distributions.
We use \(\chi^2(\pp, \pp')\) to denote the chi-squared divergence of \(\pp\) w.r.t.\ \(\pp'\), i.e., \(\chi^2(\pp, \pp') = \int (\frac{\mathrm{d}\pp}{\mathrm{d}\pp'} - 1)^2 \mathrm{d}\pp'\).

\subsection{Distributional Assumptions}\label{sec:distr-assumpt}

Similar to \cite{wang2023robustly-learning}, we make two assumptions about the target distribution of the covariates ($\pp_\vx^*$).  
First, we assume that the optimal solution $\vw^*$ satisfies the following ``margin-like'' condition:

\begin{assumption}[Margin]\label{assump:margin}
  There exist absolute constants \(\lambda, \gamma \in (0, 1]\) such that
  \(\E_{\vx \sim \pp^{*}_{\vx}}[ \vx \vx^T \Ind_{\vw^{*} \cdot \vx \ge \gamma \norm{\vw^{*}}_{2}} ]\succeq \lambda \mI,\)
  where $\pp^*_\vx$ is the $\vx$-marginal distribution of $\pp^*$. %
\end{assumption}

We also assume that $\pp_\vx^*$ is subexponential with
parameter $B$, which is an absolute constant.

\begin{assumption}[Subexponential Concentration]\label{assump:concentration}
  There exists a parameter \(B > 0\) such that for any \(\vu \in \cB(1)\) and any \(r \ge 1\), it holds that \(\Pr_{\vx \sim \pp^{*}_{\vx}}[\abs{\vu \cdot \vx} \ge r] \le \exp(-Br)\).
\end{assumption}

Appendix E of \cite{wang2023robustly-learning} shows that \Cref{assump:margin,assump:concentration} 
are satisfied by several important families of distributions including Gaussians, discrete Gaussians, all isotropic log-concave distributions, the uniform distribution over $\{-1, 0, 1\}^d$, etc.
For simplicity, we assume the labeled samples \((\vx^{(i)}, y^{(i)})\) drawn from the reference distribution are bounded. 
This assumption, which does not affect the approximation constant for \Cref{prob:main}, only impacts iteration and sample complexities. 
We state the bound on the covariates below, while a bound on the labels follows from prior work (\Cref{fact:truncation} stated in the next subsection).

\begin{assumption}[Boundedness]\label{assump:boundedness}
  There exists a parameter \(S\) such that for any fixed \(\vu \in \cB(1)\) it holds that \(\vu \cdot \vx \le S\) for all sample covariates \(\vx\) in the support of \(\ep_{0}\). 
\end{assumption}

We also assume without loss of generality that $\norm{\vw^*}_2^2 \ge C \OPT + \eps$ for some absolute constant $C$, \new{since otherwise \(\vzero\) would be a valid \(O(\OPT) + \epsilon\) solution. Algorithmically, we can first compute the empirical risk (per \Cref{cor:risk-distance-computation}) of the output from our algorithm and of $\hat \vw = \vzero$ and then output the solution with the lower risk to get an $\bigO{\OPT} + \eps$ solution; see \Cref{claim:zero-tester} for a detailed discussion.}
\subsection{Auxiliary Facts}
\label{sec:basic_facts}

To achieve the claimed guarantees, we leverage structural properties of the loss function on the target distribution, implied by our distributional assumptions (\Cref{assump:margin,assump:concentration}).  %
Specifically, we make use of Lemma 2.2 and Fact C.4 from \cite{wang2023robustly-learning}, summarized in the fact below.
\begin{fact}[Sharpness (\cite{wang2023robustly-learning})]\label{fact:sharpness}
  Suppose \(\pp^{*}\) and \(\vw^{*}\) satisfy \Cref{assump:margin,assump:concentration}. Let \(c_{0} = \frac{\gamma \lambda \alpha}{6B\log(20B/\lambda^{2})}\).  For all \(\vw \in \cB(2\norm{\vw^{*}})\) and \(\vu \in \cB(1)\),
  \begin{align*}
    \E_{\vx \sim \pp^{*}_{\vx}}[(\sigma(\vw \cdot \vx) - \sigma(\vw^{*} \cdot \vx))(\vw \cdot \vx - \vw^{*} \cdot \vx)]   & \ge c_{0} \norm{\vw - \vw^{*}}_{2}^{2},
    \\
    \E_{\vx \sim \pp^{*}_{\vx}}[ (\vx \cdot \vu)^{\tau}] & \le 5 B \quad \mbox{for } \tau = 2, 4. 
  \end{align*}
\end{fact}
\Cref{fact:sharpness} applies to the population version of the problem. 
Such a result also holds for the target distribution of the empirical problem, which we state below. 
Note that this result cannot be obtained by  appealing to uniform convergence results for learning a neuron (without distributional robustness). 

\begin{lemma}[Empirical Sharpness; Informal. See \Cref{lemma:sharpness-empirical-appendix}]\label{lemma:sharpness-empirical}
  Under \Cref{assump:margin,assump:concentration,assump:boundedness}, %
  for a sufficiently large sample size \(N\) as a function of $B, W, S, \nu, \alpha, \gamma, \lambda, d$ and with high probability, for all \(\vw \in \cB(2\norm{\vw^{*}})\) with  \(\norm{\vw - \vw^{*}} \ge \sqrt\epsilon\)  and \(\vu \in \cB(1)\),
  \begin{align}
    \E_{\vx \sim \ep^{*}_{\vx}}[(\sigma(\vw \cdot \vx) - \sigma(\vw^{*} \cdot \vx))(\vw \cdot \vx - \vw^{*} \cdot \vx)]   & \ge  (c_{0}/2)\norm{\vw - \vw^{*}}_{2}^{2} \label{eq:sharpness-empirical} \\
    \E_{\vx \sim \ep^{*}_{\vx}}[(\vx \cdot \vu)^{\tau}] & \le 6 B  \label{eq:moment-bounds-empirical} \quad \mbox{for } \tau = 2, 4.
  \end{align}
\end{lemma}
  As a consequence, for \(c_{1} =c_{0}^{2}/(24 B) \) and any $\vw \in \cB(W)$ (where $c_0$ is defined in \Cref{fact:sharpness}), we have
  \begin{equation}
    \label{eq:relu-square-bound}
    c_1 \norm{\vw - \vw^{*}}_{2}^{2} \le \E_{\vx \sim \ep_{\vx}^{*}}[(\sigma(\vw \cdot \vx) - \sigma(\vw^{*} \cdot \vx))^{2}] \le 6 B \beta^{2} \norm{\vw - \vw^{*}}_{2}^{2},
  \end{equation}
  where the left inequality uses Cauchy-Schwarz and the right inequality uses \(\beta\)-Lipschitzness of \(\sigma(\cdot)\).

\cite{wang2023robustly-learning} also showed that the labels $y$ can be assumed to be bounded without loss of generality. %
  \begin{fact}\label{fact:truncation}
    Suppose \(\pp^*\) and $\vw^*$ satisfy \Cref{assump:margin} and \Cref{assump:concentration}. Let \(y' = \sign(y) \max\{\abs{y}, M\}\) where for some sufficiently large  absolute constant \(C_{M}\) we define
    \begin{equation}
      \label{eq:y-truncation}
      M = C_{M} W B \beta \log( \beta B W / \epsilon)
    \end{equation}
    Then
    \(\E_{\pp^*} (\sigma(\vw^{*} \cdot \vx) - y')^{2} \le \E_{\pp^*}  (\sigma(\vw^{*} \cdot \vx) - y)^{2} + \epsilon = \OPT + \eps\).
  \end{fact}

  We  also make use of the following facts from convex analysis. First, let \(\phi : \mathbb{R}^N \to \mathbb{R}\) be a differentiable function and the Bregman divergence of $\phi$
  for  any \(\vx, \vy \in \mathbb{R}^N\) be defined by
  \[D_{\phi}(\vy, \vx) = \phi(\vy) - \phi(\vx) - \langle \nabla \phi(\vx), \vy - \vx \rangle.\]

\begin{restatable}{fact}{bregmanblind}\label{fact:bregman-linear-blind}
Let \(\psi(\vx) = \phi(\vx) + \langle \va, \vx \rangle + b\) for some \(\va \in \mathbb{R}^N\) and \(b \in \mathbb{R}\). Then \(D_{\psi}(\vy, \vx) = D_{\phi}(\vy, \vx)\) for all \(\vx, \vy \in \R^{N}\), i.e., the Bregman divergence is blind to the addition of affine terms to  function \(\phi\).
\end{restatable}

Second, we state the first-order necessary conditions that a local maximizer must satisfy. 
\begin{fact}[First-Order Optimality Condition]
  \label{fact:firstOrderNecessary}
  Let $\Omega$ be a closed, convex, and nonempty set and let $f: \Omega \to \R$ be continuously differentiable. %
  If $\vx^*$ is a local maximizer of $f$ on $\Omega$, then it holds that
\begin{equation}\label{eq:first-order-opt}
    \nabla f(\vx^*) \cdot (\vy - \vx^*) \le 0 \quad \mbox{for all } \vy \in \Omega.
\end{equation}
If $f$ is also concave, then \Cref{eq:first-order-opt} implies that $\vx^*$ is a global maximizer of $f$.
\end{fact}

\section{Algorithm and Convergence Analysis}\label{sec:algor-conv-analys}

In this section, we introduce our algorithm and state our main results, summarized in \Cref{thm:main-formal}. We highlight the main components of our technical approach, while most of the technical details are deferred to the appendix, due to space constraints.  %

To facilitate the presentation of results, we introduce the following auxiliary notation:
 $\ell(\vw; \vx, y) := (\sigma(\vw \cdot \vx) - y)^2 $, \(\vv(\vw;\vx, y) := 2\beta(\sigma(\vw \cdot \vx) - y) \vx\) and \(\OPThat = \E_{(\vx, y)\sim\ep^{*}}\ell(\vw^{*};\vx, y)\). %
We also note that \Cref{assump:boundedness} implies that for all samples $\{\vx_i, y_i\}$, the function \(\vw \mapsto \vv(\vw; \vx_{i}, y_{i})\) is bounded above by \(G\) and  \(\kappa\)-Lipschitz for all \(i \in [N]\) and \(\vw \in \cB(W)\), 
where \(G = 2 \beta S \sqrt{d} (\sqrt{2} \beta WS + M)\) and \( \kappa = 2 \beta^2 S^2 d\) (see \Cref{lem:bound_on_vv} in \Cref{app:supplementary}).
{Starting from this section,
we write $L(\vw, \ep)$ to denote $L_\sigma(\vw, \ep;\ep_0)$, hiding the dependence on $\ep_0$ and $\sigma$.
We also write $\Gap(\vw, \ep)$ for $\Gap(\vw, \ep;\ep_0)$}

Our main algorithm (\Cref{alg:main}) is an iterative primal-dual method with extrapolation on the primal side via \(\vg_i\). 
The vector \(\E_{\ep_i}[\vv(\vw_i; \vx, y)]\) equals the (scaled) gradient of a surrogate loss used in prior works \cite{DGKKS20,wang2023robustly-learning,kakade2011efficient}. 
In contrast to prior work, we directly bound the original square loss, with \(\E_{\ep_i}[\vv(\vw_i; \vx, y)]\) naturally arising from our analysis. 
Both updates \(\vw_i\) and \(\ep_i\) are efficiently computable: 
\(\vw_i\) involves a simple projection onto a Euclidean ball, and \(\ep_i\) involves a projection onto a probability simplex, computable in near-linear time \cite{duchi2008efficient}.

\begin{algorithm}
  \KwIn{\(\nu > 0, \kappa, G, c_{1}, \nu_0 = 768 \beta^4 B \epsilon/ c_1\), sample set \(\{(\vx_{i}, y_{i})\}_{i=1}^{N}\)}
  \textbf{Initialization:}\(A_{-1} = a_{-1} = A_{0} = a_{0} = 0, \vw_{-1} = \vw_{0} = \vzero, \ep_{-1} = \ep_{0}\)\;
  \For{\(i = 1, \dots, k\)}{
    \(a_{i} = \big(1 + \frac{\min\{\nu, c_{1}/8\}}{2\max\{\kappa, G\}}\big)^{i-1} \min\{\nu_{0}, 1/4\}/(2\max\{\kappa, G\}), A_{i} = a_{i} + A_{i-1}\)\label{line:convergence}\;

    \( \vv(\vw; \vx, y) =
       2\beta(\sigma(\vw \cdot \vx) - \sign(y) \max\{\abs{y}, M\})  \vx \), where \(M\) is defined in \Cref{eq:y-truncation} \; %

    \(\vg_{i-1} = \E_{\ep_{i-1}}[\vv(\vw_{i-1}; \vx, y)] + \frac{a_{i-1}}{a_i}(\E_{\ep_{i-1}}[\vv(\vw_{i-1}; \vx, y)] - \E_{\ep_{i-2}}[\vv(\vw_{i-2}; \vx, y)])\)\label{line:g}\;
    \(\vw_{i} = \argmin_{\vw \in \cB(W)}\big\{a_i \innp{\vg_{i-1}, \vw } + \frac{1+0.5 c_1 A_{i-1}}{2}\|\vw - \vw_{i-1}\|_2^2\)
    \label{line:w}\big\}\;
    \(\ep_{i}  = \argmax_{\ep \in \cP}\big\{ a_{i}L(\vw_{i}, \ep) - (\nu_0 + {\nu A_{i-1}} )D_{\chi^{2}(\cdot, \ep_{0})}(\ep, \ep_{i-1})\)\label{line:p}\big\}\;
  }
\caption{Main algorithm}\label{alg:main}
\end{algorithm}

\begin{theorem}[Main Theorem]\label{thm:main-formal}
   Under \Cref{assump:margin,assump:concentration,assump:boundedness}, suppose the sample size is such that \(N = \widetilde\Omega_{B, S, \beta, \alpha, \gamma, \lambda}\big(\frac{ W^{4} }{\eps^{2}}\big(1 + \frac{W^{4}}{\nu^{2}}\big)(d  + W^4 \log(1/\delta))\big)\) and \(\nu \ge  8 \beta^2 \sqrt{6B}\sqrt{\OPTT + \epsilon}/{c_1} \), where \(\OPTT  = \E_{\pp^{*}}[\ell(\vw^{*}; \vx, y)^{2}]\) \new{and $c_1$ is defined in \Cref{lemma:sharpness-empirical}}. With probability at least \(1 - \delta\), for all iterates \(\vw_{k}, \ep_k\), it holds that
     \begin{align*}
        \frac{c_1}{4}\|\vw^* - \vw_k\|_2^2 +  \nu D_\phi(\ep^*, \ep_k) \leq \frac{D_{0}}{A_{k}} +  \frac{60 \beta^2 B\OPT }{c_1} + \epsilon,
  \end{align*}
  where \(D_{0}  = \frac{1}{2}\|\vw^* - \vw_0\|_2^2 + \nu_0 \chi^2(\ep^*, \ep_0)\) and $\chi^2(\ep^*, \ep_0) \leq  c_1/(1536 \beta^4 B)$ \new{(and therefore $D_0$ does not depend on the sample size N)}.%

\new{
In particular, after at most $k = \bigtO{\frac{\max\{\kappa, G\}}{\min\{\nu, c_{1}\}} \log(\frac{D_0} \epsilon)}$ iterations, it holds that \begin{align}
     \norm{\vw_{k} - \vw^{*}}_2 & \le C_{3}\sqrt{\OPT} + \sqrt{\epsilon},\label{eq:distance-to-opt-bound}\\
     \E_{(\vx, y)\sim\pp^*}[\ell(\vw_k; \vx, y)] &\le   (2 + 20 B \beta^2 C_3^2)  ~\OPT + {10 \beta^2  B} \eps, \label{eq:square-loss-bound} \\
R(\vw_k; \pp_0) - \min_{\vw \in \cB(W)} R(\vw; \pp_0) & = R(\vw_k;\pp_{0}) - R(\vw^{*}; \pp_{0})  \le C_{4} (\OPT + \epsilon) \label{eq:risk-bound},
\end{align}
{ where }\(C_{3} = 16 \beta\sqrt{B}/c_{1}\) and \( 
    C_{4} = 1 + 2 (10 B \beta^{2} + c_{1})C_{3} + c_{1} \sqrt{5B} \beta^{2} C_{3}^{2} \).
}
\end{theorem}

\new{We focus on the convergence of iterates ${\vw_i}$ as claimed in \Cref{eq:distance-to-opt-bound}; the loss bound (\Cref{eq:square-loss-bound}) follows directly from the iterate convergence, while the risk bound (\Cref{eq:risk-bound}) requires a more involved analysis. Complete details for \Cref{eq:square-loss-bound,eq:risk-bound} are provided in \Cref{app:parameter_estimation}.}

Our strategy for the convergence analysis is as follows. 
Consider \(\{a_{i}\}\), a sequence of positive step sizes, and define \(A_{i}\) as their cumulative sum \(\sum_{j=1}^{i}a_{j}\). 
Our algorithm produces a sequence of primal-dual pairs \(\vw_{i}, \ep_{i}\), tracking a quantity related to the primal-dual gap, defined by:
\[
\Gap(\vw_{i}, \ep_{i}) := L(\vw_{i}, \ep^{*}) - L(\vw^{*}, \ep_{i}) = (L(\vw_{i}, \ep^{*}) - L(\vw^*, \ep^*)) + (L(\vw^*, \ep^*) - L(\vw^{*}, \ep_{i})).
\]
We view \((L(\vw_{i}, \ep^{*}) - L(\vw^*, \ep^*))\) as the ``primal gap'' and \((L(\vw^*, \ep^*) - L(\vw^{*}, \ep_{i}))\) as the ``dual gap.'' 
Since the squared loss for ReLU and similar activations is nonconvex, \(L(\vw, \ep^{*})\) is nonconvex in its first argument. 
Note that this gap function is not trivially non-negative (see \Cref{remark:gap_lowerbound}), requiring an explicit lower bound proof. %

Our strategy consists of deriving ``sandwiching'' inequalities for the (weighted) cumulative gap $\sum_{i=1}^{k}a_{i}\Gap(\vw_{i}, \ep_{i})$ and deducing convergence guarantees for the algorithm iterates from them. %
A combination of these two inequalities leads to the statement of Theorem \ref{thm:main-formal}, from which we can deduce that unless we already have an $O(\OPT) + \epsilon$ solution, the iterates must be converging to the target solutions at rate $1/A_k$, which we argue can be made geometrically fast. 

\noindent\textbf{Organization} \quad
The rest of this section is organized as follows --- under the standard assumptions we state in this paper, in \Cref{thm:gap-lower-bound}, we prove a lower bound on $\Gap(\vw, \hat \pp)$ for any choice of $\vw$ and $\hat \pp$.
This can be used to get a corresponding lower bound on the weighted sum $\sum_{i=1}^k a_i \Gap(\vw_i, \hat\vp_i)$. 
In \Cref{lemma:gap-upper-bound} we then state an upper bound on $\sum_{i=1}^k a_i \Gap(\vw_i, \hat\vp_i)$; the proof of this technical argument is deferred to \Cref{app:gap-upper-bound}.
These two bounds together give us the first inequality in \Cref{thm:main-formal}.
\Cref{thm:convergence-rate} then bounds below the convergence rate for our choice of $a_i$ in \Cref{alg:main}; and indicates that it is geometric.
Finally, we put everything together to prove \Cref{thm:main-formal}. 

To simplify the notation, we use \(\phi(\ep) := \chi^{2}(\ep, \ep_{0})\) throughout this section. Note that \(D_{\phi}(\ep, \eq) = D_{\phi}(\eq, \ep) = \sum_{i=1}^N \frac{(\eq^{(i)} - \ep^{(i)})^2}{\ep_{0}^{(i)}}\) for any \(\ep\) and \(\eq\) in the domain. %

\subsection{Lower Bound on the Gap Function}
We begin the convergence analysis by demonstrating a lower bound on \(\Gap(\vw_{i} ,\ep_{i})\).

\begin{lemma}[Gap Lower Bound]\label{thm:gap-lower-bound}
  Under the setting in which \Cref{lemma:sharpness-empirical} holds,  for all \(\vw \in \cB(2\norm{\vw^{*}}_{2})\),
  \(\Gap(\vw ,\ep) \ge -\frac{12\beta^2 {B}}{c_1}\OPThat + \frac{c_1}{2} \|\vw - \vw^{*}\|_{2}^{2} + \nu D_{\phi}(\ep^*, \ep)%
  .\)
\end{lemma}
\begin{proof}
  Writing $(\sigma(\vw \cdot \vx) - y)^{2} = \big((\sigma(\vw \cdot \vx) - \sigma(\vw^{*} \cdot \vx)) + (\sigma(\vw^{*} \cdot \vx) - y)\big)^{2}$ and expanding the square, we have %
\begin{align*}
  & \;L(\vw, \ep^{*}) - L(\vw^{*}, \ep^{*}) = \mathbb{E}_{(\vx, y) \sim \ep^{*}}[(\sigma(\vw \cdot \vx) - y)^{2} - (\sigma(\vw^{*} \cdot \vx) - y)^{2}] \\
  = & -2\mathbb{E}_{\ep^{*}}[(\sigma(\vw^{*} \cdot \vx) - y)(\sigma(\vw \cdot \vx) - \sigma(\vw^{*} \cdot \vx))] + \mathbb{E}_{\ep^{*}}[((\sigma(\vw \cdot \vx) - \sigma(\vw^{*} \cdot \vx))^{2}].  %
\end{align*}

By the Cauchy-Schwarz inequality, we further have that
\begin{align}\label{eq:lb-gap-CS}
   &\; \mathbb{E}_{\ep^{*}}[(\sigma(\vw^{*} \cdot \vx) - y)
   (\sigma(\vw \cdot \vx) - \sigma(\vw^{*} \cdot \vx))]\notag\\
   \leq \; &\sqrt{\mathbb{E}_{\ep^{*}}[(\sigma(\vw^{*} \cdot \vx) - y)^2]\mathbb{E}_{\ep^{*}}[(\sigma(\vw \cdot \vx) - \sigma(\vw^{*} \cdot \vx))^2]}\notag\\
    \leq\; & \beta\sqrt{6 B} \sqrt{\OPThat }\|\vw - \vw^*\|_2,
\end{align}
where in the second inequality we used the definition of $\OPThat $ and
$\E_{\vx \sim \ep_{\vx}^{*}}[(\sigma(\vw \cdot \vx) - \sigma(\vw^{*} \cdot \vx))^{2}] \le 6 B \beta^{2} \norm{\vw - \vw^{*}}_{2}^{2}$ from the right inequality in \Cref{eq:relu-square-bound}. %

On the other hand, by the left inequality in \Cref{eq:relu-square-bound}, we also have 
\begin{equation}\label{eq:lb-gap-sharpness}
    \mathbb{E}_{\ep^{*}}[(\sigma(\vw \cdot \vx) - \sigma(\vw^{*} \cdot \vx))^{2}] \geq c_1 \|\vw - \vw^*\|_2^2.
\end{equation}

Thus, combining \Cref{eq:lb-gap-CS} and \Cref{eq:lb-gap-sharpness}, we get 
\begin{align}
  L(\vw, \ep^{*}) - L(\vw^{*}, \ep^{*})
  &\ge - 2 \beta \sqrt{6 B} \|\vw - \vw^{*}\|_{2} \sqrt{\OPThat }  + c_1 \|\vw - \vw^{*}\|_{2}^{2}\notag\\
  & {\ge -\frac{12 \beta^{2} B}{c_{1}} \OPThat + \frac{c_1}{2} \|\vw - \vw^{*}\|_{2}^{2}},\label{eq:lb-gap-last1}
\end{align}
where the last inequality is by $2 \beta \sqrt{6 B} \|\vw - \vw^{*}\|_{2} \sqrt{\OPThat } \leq \frac{4\beta^2 {6B}}{2c_1}\OPThat  + \frac{c_1}{2} \|\vw - \vw^{*}\|_{2}^{2}$, which comes from an application of Young's inequality (\Cref{fact:young}).

Finally, we use the optimality of \(\ep^{*}\), which achieves the maximum over all \(\ep \in \cP\) for \(L(\vw^{*}, \ep)\). By the definition of a Bregman divergence, \Cref{fact:bregman-linear-blind}, and first-order necessary condition in \Cref{fact:firstOrderNecessary}:
\begin{equation}\label{eq:lb-gap-last2}
    -L(\vw^{*}, \ep) - (-L(\vw^{*}, \ep^{*})) = -\innp{ \nabla_{\ep} L(\vw^{*}, \ep^{*}), \ep - \ep^{*}} + D_{-L(\vw^{*}, \cdot)}(\ep, \ep^{*}) \ge \nu D_\phi(\ep^*, \ep).
\end{equation}
Summing up \Cref{eq:lb-gap-last1} and \Cref{eq:lb-gap-last2} completes the proof.
\end{proof}

\subsection{Upper Bound on the Gap Function}

Having obtained a lower bound on the gap function, we now show an upper bound, leveraging our algorithmic choices. The proof is rather technical and involves individually bounding $L(\vw_{i}, \ep^{*})$ above and bounding $L(\vw^{*}, \ep_{i})$ below to obtain an upper bound on the gap function, which equals $L(\vw_{i}, \ep^{*}) - L(\vw^{*}, \ep_{i})$. We state this result in the next lemma, while %
the proof is in \Cref{app:gap-upper-bound}. 
\begin{restatable}[Gap Upper Bound]{lemma}{lemgapub}\label{lemma:gap-upper-bound}
    Let $\vw_i, \ep_i, a_i, A_i$ evolve according to \Cref{alg:main}, where we take, by convention, $a_{-1} = A_{-1} = a_0 = A_0 = 0$ and $\vw_{-1} = \vw_0,$ $\ep_{-1} = \ep_0.$ Assuming \Cref{lemma:sharpness-empirical} applies, then, for all $k \geq 1,$ $ \sum_{i=1}^k a_i \Gap(\vw_i, \ep_i)$ is bounded above by
    \begin{align*}
     \; & \frac{1}{2}\|\vw^* - \vw_0\|_2^2 + \nu_0 D_\phi(\ep^*, \ep_0)
         - \frac{1 + 0.5 c_1A_{k}}{2}\|\vw^* - \vw_k\|_2^2
         - (\nu_0 + \nu A_k)D_\phi(\ep^*, \ep_k) \\
         & + \sum_{i=1}^k a_i \frac{c_1}{4} \|\vw^* - \vw_i\|_2^2 + \frac{8\beta^2 \sqrt{6B}\sqrt{\OPTThat}}{c_1}\sum_{i=1}^k a_i \chi^2(\ep_i, \ep^*) + \frac{48 \beta^2 B\OPThat A_k}{c_1}. 
    \end{align*}
\end{restatable}
{ A critical technical component in the proof of \Cref{lemma:gap-upper-bound} is how we handle issues related to nonconvexity. A key technical result that we prove and use is the following. 
\begin{lemma}
\label{lemma:main-auxiliary-main-body}
Let 
$S_i := \E_{\ep_{i}}[(\sigma(\vw^{*} \cdot \vx) - \sigma(\vw_{i} \cdot \vx))^2] + \E_{\ep_{i}}[2(\sigma(\vw_{i} \cdot \vx) - y)(\sigma(\vw^{*} \cdot \vx) - \sigma(\vw_{i} \cdot \vx))]$, $\vw_i$ evolve according to \Cref{line:w} in \Cref{alg:main} and suppose we are in the setting where \Cref{lemma:sharpness-empirical} holds.
Then, $S_i \geq \E_{\ep_{i}}[\innp{\vv(\vw; \vx, y), \vw^* - \vw_i}] - E_i$
    where
\begin{equation}\label{eq:Ei-def-mb}
        E_i = \frac{c_1}{4}\norm{\vw^* - \vw_i}_2^2 + \Big({8\beta^2 \sqrt{6B}\sqrt{\OPTThat}}/{ c_1}\Big)\chi^2(\ep_i, \ep^*) + {(48 \beta^2 B/{c_1}) \OPThat}.%
    \end{equation}
\end{lemma}

This bound is precisely what forces us to choose chi-squared as the measure of divergence between distributions and introduce a dependence on $\OPTThat$. 
One pathway to generalize our results to other divergences would be to find a corresponding generalization to \Cref{lemma:main-auxiliary-main-body}.
}

\subsection{Proof of Main Theorem}

Combining \Cref{thm:gap-lower-bound} and \Cref{lemma:gap-upper-bound}, we are now ready to prove our main result.

\begin{proof}[Proof of \Cref{thm:main-formal}]
 
    Combining the lower bound on the gap function from \Cref{thm:gap-lower-bound} with the upper bound from \Cref{lemma:gap-upper-bound} and rearranging, whenever \(\norm{\vw_{i}}_{2} \le 2\norm{\vw^{*}}_{2}\) for all \(i \le k\) so that \Cref{lemma:sharpness-empirical} applies, we get that 
    \begin{align*}
      & \; -\frac{12\beta^2 {B}}{c_1}\OPThat A_{k} +  \sum_{i=1}^k a_{i}\frac{c_1}{2} \|\vw_{i} - \vw^{*}\|_{2}^{2} +  \sum_{i=1}^k \nu a_{i}D_{\phi}(\ep^*, \ep_{i}) \le  \sum_{i=1}^k a_i \Gap(\vw_i, \ep_i)\\
       \leq \;& \frac{1}{2}\|\vw^* - \vw_0\|_2^2 + \nu_0 D_\phi(\ep^*, \ep_0)
         - \frac{1 + 0.5 c_1A_{k}}{2}\|\vw^* - \vw_k\|_2^2 - (\nu_0 + \nu A_k)D_\phi(\ep^*, \ep_k)\\
         &+ \sum_{i=1}^k a_i \frac{c_1}{4} \|\vw^* - \vw_i\|_2^2 + \frac{8\beta^2 \sqrt{6B}\sqrt{\OPTThat}}{c_1}\sum_{i=1}^k a_i \chi^2(\ep_i, \ep^*) + \frac{48 \beta^2 B\OPThat A_k}{c_1}.
    \end{align*}
    To reach the first claim of the theorem, we first argue that $\sum_{i=1}^k a_i \Big(({4\beta^2 \sqrt{6B}\sqrt{\OPTThat}}/{c_1}) \chi^2(\ep_i, \ep^*) - \nu D_\phi(\ep^*, \ep_i)\Big) \leq 0.$ This follows from (1) \Cref{cor:key_fonc-no-max}, by which we have \({\ep^{*(j)}}  \ge {\ep_{0}^{(j)}}/2 \) for all \(j \in [N]\), hence
  \begin{align*}
  \chi^{2}(\ep_{i}, \ep^{*}) = \sum_{j\in [N]}{({\ep^{*(j)} }- {\ep_{i}}^{(j)})^{2}}/{{\ep^{*(j)}}} \leq 2 \sum_{j\in [N]}{({\ep^{* (j)} }- {\ep_{i}}^{(j)})^{2}}/{\ep_{0}^{(j)}} = 2 D_{\phi}(\ep^{*}, \ep_{i})
  \end{align*}
  and (2) our choice of $\nu,$ which ensures, with high probability, that  $\nu \geq 8 \beta^2 \sqrt{6B}\sqrt{\OPTT + \epsilon}/{c_1} \ge 8 \beta^2 \sqrt{6B}\sqrt{\OPTThat}/{c_1}$, where the last inequality is because for the specified sample size, we have that $\OPTThat + \epsilon \geq \OPTT$  by \Cref{cor:relationships-for-OPTs}.

  Second,  we similarly have that with probability $1- \delta,$ $\OPThat \leq \OPT + \epsilon.$  Hence, since Bregman divergence of a convex function is non-negative, whenever \(\norm{\vw_{i}}_{2} \le 2\norm{\vw^{*}}_{2}\) for all \(i\le k\), we have
  \begin{align}
    \|\vw^* - \vw_k\|_2^2 & \leq \frac{\|\vw^* - \vw_0\|_2^2 + 2\nu_0 D_{\phi}(\ep^*, \ep_0)}{1 + 0.5c_1 A_k} + \frac{240 \beta^2 B}{c_1}(\OPT + \epsilon) \\
    D_\phi(\ep^*, \ep_k) & \leq \frac{\|\vw^* - \vw_0\|_2^2/2 + \nu_0 D_{\phi}(\ep^*, \ep_0)}{\nu_0 + \nu A_k}+  \frac{60 \beta^2B}{\nu}(\OPT + \epsilon)
  \end{align}
  The bound $\chi^2(\ep^*, \ep_0) \leq  c_1/(1536 \beta^4 B)$ is proved in \Cref{claim:ambiguity-radius}. 
  Finally, in \Cref{sec:bounded-iterate}, we inductively prove the following claim so that  assumptions in \Cref{lemma:sharpness-empirical} are satisfied.
  \begin{restatable}{claim}{boundediterate}\label{claim:bounded-iterate}
      For all iterations \(k \ge 0\), \(\norm{\vw_{k}}_{2} \le 2\norm{\vw^{*}}_{2}\) .
  \end{restatable}

    The bound on the growth of $A_k$ follows by standard arguments and is provided as \Cref{thm:convergence-rate}.  Since $A_k$ grows exponentially with $(1+\eta)^k $ where \(\eta = \frac{\min\{\nu, c_{1}/8\}}{2\max\{\kappa, G\}}\) and since $D_{0}(1 + \eta)^{-k} \le \epsilon $ can be enforced by setting $k = (1 + 1/\eta) \log(D_0 / \epsilon) \ge \log(D_0 / \epsilon) / \log(1 + \eta)$, we have that after $\bigtO{\frac{\max\{\kappa, G\}}{\min\{\nu, c_{1}\}} \log(D_0 / \epsilon)}$ iterations either $\|\vw_{k} - \vw^{*}\|_{2}\le \sqrt\epsilon$ or $\|\vw_{i} - \vw^{*}\|_{2} \le  C_{3}\sqrt{\OPT}$.
\end{proof}

\section{Conclusion}
In this paper, we study the problem of learning a single neuron in the distributionally robust setting, with the square loss regularized by the chi-squared distance between the reference and target distributions. 
Our results serve as a preliminary exploration in this area, paving the way for several potential extensions. Future work includes generalizing our approach to single index models with unknown activations, expanding to neural networks comprising multiple neurons, and considering alternative ambiguity sets such as those based on the Wasserstein distance or Kullback-Leibler divergence.

\section*{Acknowledgement}
Shuyao Li was supported in part by AFOSR Award FA9550-21-1-0084 and the U.S.\ Office of Naval Research under award number N00014-22-1-2348.
Sushrut Karmalkar was supported by NSF under Grant \#2127309 to the Computing Research Association for the CIFellows 2021 Project.
Ilias Diakonikolas was supported by NSF Medium Award CCF-2107079 and an H.I. Romnes Faculty Fellowship.
Jelena Diakonikolas was supported in part by the U.S.\ Office of Naval Research under contract number N00014-22-1-2348. Any opinions, findings and conclusions or recommendations expressed in this material are those of the author(s) and do not necessarily reflect the views of the U.S. Department of Defense.

\clearpage
\newrefcontext[sorting=nyt]
\printbibliography
\clearpage
\appendix
  \crefalias{section}{appendix}

\begin{center}{\LARGE\bfseries Supplementary Material}
\end{center}

\paragraph{Organization} 
In \Cref{sec:related_work} we briefly discuss related work. 
In \Cref{app:supplementary} we set up some additional preliminaries for the rest of the appendix. 
In \Cref{app:concentration_pp_star} we show that expectations of some important functions with respect to $\ep^*$ are close to their expectation with respect to $\pp^*$.
In \Cref{app:gap-upper-bound} we give a detailed proof of an upper bound on the gap of the iterates our algorithm generates (i.e. \Cref{lemma:gap-upper-bound}). 
Finally, in \Cref{app:parameter_estimation} we show that the estimate of $\vw^*$ our algorithm returns is a constant factor approximation to the squared loss of $\vw^*$ with respect to the target distribution. 

\section{Related Work}
\label{sec:related_work}

\paragraph{Learning Noisy Neurons}
Generalized linear models are classical in statistics and machine learning~\cite{nelder1972generalized}. 
The problem of learning noisy neurons has been extensively explored in the past couple of decades; 
notable early works include \cite{kalai2009isotron, kakade2011efficient}. 
In the recent past, the focus has shifted towards specific activation functions such as ReLUs, 
under both easy noise models such as realizable/random additive noise \cite{soltanolkotabi2017learning, kalan2019fitting, yehudai2020learning} 
and more challenging ones, including 
semi-random and adversarial label noise \cite{GoelKK19, DKZ20-sq-reg, GGK20, DKPZ21-SQ, diakonikolas2021relu, DGKKS20, diakonikolas22a-hardness, DKTZ22, wang2023robustly-learning, ZWDD2024}.

Even with clean labels, this problem has exponentially many local minima when using squared loss \cite{Auer95}.
{Unfortunately, directly minimizing the squared loss using (S)GD on a bounded distribution does not converge to the global optimum with probability 1 \cite{yehudai2020learning}.
Even so, gradient based methods can achieve suboptimal rates in the agnostic setting %
for distributions with mild distributional assumptions \cite{frei2020agnostic}.}
Making slightly stronger assumptions on the marginal does allow us to get efficient constant factor approximations. \cite{DGKKS20} developed an efficient learning method that is able to handle this in the presence of adversarial label noise and for isotropic logconcave distributions of the covariates. 
This was later extended to broader classes of activation functions and under weaker distributional assumptions by \cite{DKTZ22, wang2023robustly-learning}. 
Without specific distributional assumptions, learning remains computationally difficult \cite{diakonikolas22a-hardness}. 
The challenges extend to distribution-free scenarios with semi-random label noise, where methods like those in \cite{diakonikolas2021relu} address bounded noise, and \cite{karmakar2020study} and \cite{chen2020classification} explore stricter forms of Massart noise in learning a neuron. %
In this paper, we consider the harder setting of distributionally robust optimization, %
where an adversary is allowed to impose not only errors in the labels, but also adversarial shifts in the underlying distribution of the covariates.

\paragraph{Distributionally Robust Optimization}
Distributional mismatches in data have been extensively studied in the context of learning from noisy data. This includes covariate shift, where the marginal distributions might be perturbed, \cite{shimodaira2000improving, huang2006correcting, bickel2007discriminative}, 
and changes in label proportions \cite{dwork2012fairness, xu2020class}. 
This research also extends to domain adaptation and transfer learning \cite{mansour2009domain, ben2010theory, patel2015visual, pan2009survey, tan2018survey}. Distributionally robust optimization (DRO) has a rich history in optimization \cite{ben2009robust, shapiro2017distributionally} and has gained traction in machine learning \cite{namkoong2016stochastic, duchi2021learning, duchi2021statistics, staib2019distributionally, kuhn2019wasserstein, zhu2020kernel}, showing mixed success across applications like language modeling \cite{oren2019distributionally}, class imbalance correction \cite{xu2020class}, and group fairness \cite{hashimoto2018fairness}.

Specifically this has also been studied in the context of linear regression and other function approximation \cite{blanchet2021statistical,duchi2021learning,chen2018robust}.
Typically, DRO is often very sensitive to additional sources of noise, such as outliers (\cite{zhai2021doro,hashimoto18a,hu2018does}). However, prior work makes strong assumptions on the label noise as well as requiring convexity of the loss.
We study the problem of learning a neuron where the labels have no guaranteed structure, effectively studying the setting for a combination of two notions of robustness --- agnostic learning as well as covariate shift.

\section{Supplementary Preliminaries}
\label{app:supplementary}
\subsection{Additional Notation}
Given an \(m \times n\) matrix \(\mathbf{A}\),  the operator norm of \(\mathbf{A}\) is defined in the usual way as \(\|\mathbf{A}\|_{\rm op} = \sup\{\|\mathbf{A} \vx\|_2: \vx \in \mathbb{R}^n, \|\vx\|_2 \le 1\} \). For %
problems $(P)$ and $(P')$, we use $ (P) \equiv (P')$ to denote the equivalence of $(P)$ and $(P')$. \new{For a vector space \(\mathbb{E}\), we use \(\mathbb{E}^{*}\) to denote its dual space.}

\subsection{Standard Facts and Proofs}
\begin{fact}[Young's inequality]\label{fact:young}
If $a \geq 0$ and $b \geq 0$ are {nonnegative real numbers} and if $p > 1$ and $q > 1$ are real numbers such that
\[
\frac{1}{p} + \frac{1}{q} = 1,
\]
then
\[
ab \leq \frac{a^p}{p} + \frac{b^q}{q}.
\]
Equality holds if and only if $a^p = b^q$.
\end{fact}
\begin{fact}[Hoeffding's Inequality]
\label{lem:hoeffding}
Let \(X_1, X_2, \ldots, X_n\) be independent random variables such that \(a_i \le X_i \le b_i\) almost surely for all \(i\). Let \(\overline{X} = \frac{1}{n} \sum_{i=1}^{n} X_i\). Then, for any \(t > 0\),

\[
\Pr\left[\abs{\overline{X} - \E[\overline{X}]} \ge t\right] \le 2 \exp\left(-\frac{2n t^2}{\sum_{i=1}^{n} (b_i - a_i)^2}\right).
\]
\end{fact}

\bregmanblind*

\begin{proof}[Proof of \Cref{fact:bregman-linear-blind}]
  The Bregman divergence \(D_{\phi}\) and \(D_{\psi}\) are defined  by:
\[
D_{\phi}(\vy, \vx) = \phi(\vy) - \phi(\vx) - \langle \nabla \phi(\vx), \vy - \vx \rangle,
\]
\[
D_{\psi}(\vy, \vx) = \psi(\vy) - \psi(\vx) - \langle \nabla \psi(\vx), \vy - \vx \rangle.
\]

Since \(\nabla \psi(\vx) = \nabla \phi(\vx) + \va\), substituting in the definition gives:
\[
D_{\psi}(\vy, \vx) = \phi(\vy) + \langle \va, \vy \rangle + b - (\phi(\vx) + \langle \va, \vx \rangle + b) - \langle \nabla \phi(\vx) + \va, \vy - \vx \rangle
\]
\[
= \phi(\vy) - \phi(\vx) - \langle \nabla \phi(\vx), \vy - \vx \rangle = D_{\phi}(\vy, \vx).
\]

Thus, the Bregman divergence is blind to the addition of linear terms to the function \(\phi\).
\end{proof}

\subsection{Auxiliary Facts}\label{sec:computation}

We first state and prove \Cref{lem:bounded_samples} to obtain upper bounds on the norm of each point, projections onto vectors of norm at most $W$, and the loss value at each point.
\begin{lemma}[Boundedness]
  \label{lem:bounded_samples}
    Fix \(\vw \in \cB(W)\). For all samples $(\vx_i, y_i)$ with truncated labels $\abs{y_i} < M$ as per \Cref{fact:truncation} and bounded covariates as per \Cref{assump:boundedness}, it holds that
    \begin{align}
        \vw \cdot \vx_{i} &\le W S \label{eq:dot-product-bound}\\
      \norm{\vx_{i}}_2 &\le S\sqrt{d} \label{eq:x-norm-bound} \\
       (\sigma(\vw \cdot \vx_{i}) - y_{i})^{2} &\le 2\beta^{2 }W^{2}(S^{2} + C_{M}^{2}B^{2} \log^{2}( W B \beta/ \epsilon)) \label{eq:loss-bound}
    \end{align}
  \end{lemma}
  \begin{proof}
\Cref{eq:dot-product-bound} follows  from \Cref{assump:boundedness}, as \[\vw \cdot \vx_{i} = \|\vw\|_2 \frac{\vw}{\|\vw\|_2} \cdot \vx_{i} \leq \|\vw\|_2 S \leq W S.\] 

 To prove \Cref{eq:x-norm-bound}, for each coordinate $ j \in [d]$, we have \(\abs{\vx^{(j)}} = \sign(\vx^{(j)})\ve_j\cdot \vx \leq  S\), by again using \Cref{assump:boundedness}. Therefore, $\norm{\vx}_2 \le S \sqrt{d}$.

For \Cref{eq:loss-bound}, we recall \Cref{fact:truncation} that for some sufficiently large absolute constant $C_M$, it holds that 
\(\abs{y} \le M := C_{M} W B \beta \log(\beta B W / \epsilon)\). 
Thus, \Cref{eq:loss-bound} follows from Young's inequality (\Cref{fact:young}) and \Cref{eq:x-norm-bound}, since $|\sigma(t)| \leq \beta |t|$, which follows from $\beta$-Lipschitzness of $\sigma$ and $\sigma(0) = 0$.
  \end{proof}
  
\begin{lemma}[Bounds on $\vv$]
\label{lem:bound_on_vv}
  Let  
  \[ \vv(\vw; \vx, y) =
       2(\sigma(\vw \cdot \vx) - y) \beta \vx
  \]
  Then $\vv$ is uniformly bounded by $G$ in \(\ell_{2}\)-norm and $\kappa$-Lipschitz for all samples $\vx, y$ with truncated labels $\abs{y} < M$ and bounded covariates as per \Cref{assump:boundedness}, where
    \(G = 2 \beta S \sqrt{d} (\sqrt{2} \beta WS + M)\) and \( \kappa = 2 \beta^2 S^2 d\).
\end{lemma}

\begin{proof}

We first uniformly upper bound $\vv$. An application of \Cref{lem:bounded_samples} gives us,
\begin{align*}
\| \vv(\vw; \vx, y) \|_2^2 &=
4\beta^2 ~(\sigma(\vw \cdot \vx) - y)^2 ~\|\vx\|^2 \\
&\leq 4~\beta^2~(2 \beta^2 W^2 S^2 + M^2)~S^2 d.
\end{align*}
Taking square roots, we get $\| \vv(\vw; \vx, y) \|_2 \leq 2 \beta S \sqrt{d} (\sqrt{2} \beta WS + M) =: G$. 

We now upper bound the Lipschitz constant $\kappa$. We will use the fact that $\sigma$ is $\beta$-Lipschitz. 
\begin{align*}
\| \nabla_\vw \vv(\vw; \vx, y) \|_2 &= 2\beta |\sigma'(\vw \cdot \vx)| \| \vx \vx^T\|_2 \\
& = 2 \beta |\sigma'(\vw \cdot \vx)| \|\vx\|_2^2\\
& = 2 \beta \cdot \beta ~S^2d = 2 \beta^2 S^2 d =: \kappa.
\end{align*}
\end{proof}
    
\begin{corollary}
\label{lemma:boundedness-to-norm}
  Fix a reference distribution \(\pp_{0}\). Suppose \(\norm{\vv(\vw; \vx, y)}_{2} \le G\) for all \(\vw\) almost surely. Then for all distributions \(\pp, \pq \in \cP(\pp_{0})\)  it holds that
  \[\norm{\E_{\pp}[\vv(\vw; \vx, y)] - \E_{\pq}[\vv(\vw; \vx, y)]}_{2}^{2} \le G^{2} D_{\phi}(\pp, \pq).\]
\end{corollary}

\begin{proof}
  \begin{align*}
    \norm{\E_{\pp}[\vv(\vw; \vx, y)] - \E_{\pq}[\vv(\vw; \vx, y)]}_{2}^{2} %
    = \; & \Bigl\|\int \vv(\vw; \vx, y) (\dd \pp - \dd \pq )\Bigr\|^{2}_{2} \\
    = \; & \Bigl\|\int \vv(\vw) \Bigl(\frac{\dd \pp}{\dd \pp_{0}} - \frac{\dd \pq}{\dd \pp_{0}} \Bigr) \dd \pp_{0}\Bigr\|^{2}_{2}\\
    \overset{(i)}\le \; & \int \Bigl\| \vv(\vw) \Bigl(\frac{\dd \pp}{\dd \pp_{0}} - \frac{\dd \pq}{\dd \pp_{0}} \Bigr)\Bigl\|_{2}^{2} \dd \pp_{0} \\
    \overset{(ii)}\le\;  & G^{2} \int \Bigl(\frac{\dd \pp}{\dd \pp_{0}} - \frac{\dd \pq}{\dd \pp_{0}} \Bigr)^{2}\dd \pp_{0} \\
    =\; & G^{2}D_{\chi^{2}(\cdot, \pp_0)}(\pp, \pq),
  \end{align*}
  where \((i)\) is an application of Jensen's inquality and \((ii)\) follows from \Cref{lem:bound_on_vv}. 
\end{proof}
    
\begin{claim}[Convergence Rate]\label{thm:convergence-rate}
For all \(i \ge 0\), let \(a_{i}\) be defined as in \Cref{line:convergence}. Then it holds that $\frac{2G^2{a_i}^2}{1 +  0.5 c_{1} A_i} \leq \nu_0 + \nu A_{i-1}$ and $\frac{2\kappa^2{a_i}^2}{1 +  0.5 c_{1} A_i} \leq \frac{1 +  0.5 c_{1} A_{i-1}}{4}$ for all \(i\). Moreover, \(A_{k} = \sum_{i=0}^{k}a_{i} =  ((1 + \frac{\min\{\nu, c_{1}/8\}}{2\max\{\kappa, G\}})^{k}-1) \min\{\nu_{0}, 1/4\}/\min\{\nu, c_{1}/8\}\).
\end{claim}

\begin{proof}
  In order for both $\frac{2G^2{a_i}^2}{1 +  0.5 c_{1} A_i} \leq \nu_0 + \nu A_{i-1}$ and $\frac{2\kappa^2{a_i}^2}{1 +  0.5 c_{1} A_i} \leq \frac{1 +  0.5 c_{1} A_{i-1}}{4}$ to hold for all iterations \(i\), it suffices that
  \[ \frac{4\max\{G,\kappa\}^2{a_i}^2}{1 +  0.5 c_{1} A_i} \leq \min\{\nu_0, 1/4\} + \min\{\nu, c_{1}/8\} A_{i-1},\]
  for which it suffices to enforce
  \[ {4\max\{G,\kappa\}^2{a_i}^2} = (\min\{\nu_0, 1/4\} + \min\{\nu, c_{1}/8\} A_{i-1})^{2},\]
  where we used \(A_{i-1} \le A_{i}\).

  Taking a square root on both sides using \(a_{i} > 0\), we obtain
  \[{2\max\{G,\kappa\}{a_i}} = \min\{\nu_0, 1/4\} + \min\{\nu, c_{1}/8\} A_{i-1}.\]

  Solving this recurrence relation using Mathematica, we compute that for all iterations \(i\) and \(k\),
  \begin{align*}
    a_{i} &= \Bigl(1 + \frac{\min\{\nu, c_{1}/8\}}{2\max\{\kappa, G\}}\Bigr)^{i-1} \min\{\nu_{0}, 1/4\}/(2\max\{\kappa, G\}) \\
    A_{k} & = \sum_{i=0}^{k}a_{i} =  ((1 + \frac{\min\{\nu, c_{1}/8\}}{2\max\{\kappa, G\}})^{k}-1) \min\{\nu_{0}, 1/4\}/\min\{\nu, c_{1}/8\}.
  \end{align*}
\end{proof}

\begin{remark}
    \label{remark:gap_lowerbound}
    {

Note that in our case, the gap is not guarenteed to be non-negative as is usually the case for convex-concave min-max problems. 
Recall that $\Gap(\vw, \ep) = (L(\vw, \ep^*) - L(\vw^*, \ep^*))+ (L(\vw^*, \ep^*) - L(\vw^*, \ep))$ 

Consider the following example:

Let $\pp_0$ be the uniform distribution over $\{ (-2, 2), (2, 1.5) \}$, $\sigma \equiv \text{ReLU}$ and $\nu = 0$. 
Then, $\vw^* = -1$ and $\pp^*$ is the distribution which places all its mass on $(2, 1.5)$. 
Then,
$\Gap(1, \pp^*) = L(1, \pp^*) - L(-1, \pp^*) = 0.25-2.25 < 0$.

This is why we also need an explicit lower bound on the Gap that we have shown in \Cref{thm:gap-lower-bound}.
}

\end{remark}

\section{Concentration}
\label{app:concentration_pp_star}
Recall that that $\eq_\vw$ is not guaranteed to act as an empirical estimate of $\pq_\vw$, because we cannot draw samples from the (unknown) distribution $\pq_\vw$ but only from $\pp^{0}$.
In this section, we show that for certain important functions \(f\), it holds that $\E_{\eq_\vw}[f] \approx \E_{\pq_\vw}[f]$. We will abuse terminology and say that $f$ ``concentrates'' with respect to $\pq_\vw$.

\paragraph{Organization:} In \Cref{subsec:closed-form} we derive closed-form expressions for $\pq_\vw$ and $\eq_\vw$ in terms of $\pp_0$ and $\ep_0$ respectively. 
Note that bounded functions concentrate with respect to \(\pp_{0}\).
In \Cref{subsec:concentration-ep} we use the closed-form expressions found in \Cref{subsec:closed-form} to translate  these concentration properties to $\eq_\vw$.
Finally, in \Cref{subsec:emp_sharpness_loss} we show that $\ep^{*}$ satisfies sharpness, $\OPThat \approx \OPT$ and $\OPTThat \approx \OPTT $.

\subsection{Closed-form expression} 
\label{subsec:closed-form}
The following lemma gives us a closed-form expression for 
$\pq_\vw$ and $\eq_\vw$ in terms of $\pp_0$ and $\ep_0$, respectively. 
We start with an additional definition to \Cref{def:loss_risk_opt}:
\[R(\vw; \ep_{0}) := \max_{\ep \in \cP}\E_{(\vx, y) \sim \ep} (\sigma(\vw \cdot \vx) - y)^{2} - \nu \chi^2(\ep, \ep_{0}), \, \text{with the maximum achieved by } \eq_{\vw},\]
\begin{lemma}[Closed-form $\pq_\vw$]
  \label{lemma:key_fonc}
 Let \(\pp_{0}\) be a fixed distribution. Then, there exists \(\xi \in \R\) such that,
  \begin{align*}
    \frac{\dd \pq_\vw}{ \dd \pp_{0}}(\vx, y) &= \frac{\max\{\ell(\vw; \vx, y) - \xi + 2\nu, 0\}}{2\nu}.
  \end{align*}
  When \(\pp_{0}\) is the empirical distribution \(\ep_{0}(N)\), this result implies that there exists \(\hat\xi \in \R\) such that
  \begin{align*}
     {\eq_\vw}^{(i)} &= \ep_{0}^{(i)}\frac{\max\{\ell(\vw; \vx_{i}, y_{i}) - \hat\xi + 2\nu, 0\}}{2\nu}   \quad \mbox{for all } i \in [N].
  \end{align*}
   The constants \(\xi\) and \(\hat\xi\) can be interpreted as normalization that ensures \(\int \dd \pq_\vw = \int \dd \eq_\vw = 1\).
\end{lemma}
\begin{proof}
  Recall that the dual feasible set is given by  \(\cP = \cP(\pp_{0}) = \{\pp \ll \pp_{0}: \int \dd \pp = 1, \pp \ge 0\}   = \{\pp \ll \pp_{0}: \int \frac{\dd \pp}{\dd \pp_{0}} \dd \pp_{0} = 1, \frac{\dd \pp}{\dd \pp_{0}} \ge 0\}\) and the function $\pp \mapsto L(\vw, \pp)$ is strongly concave.

  Consider the following optimization problem
  \[\max_{\pp \in \cP(\pp_{0})} L(\vw, \pp) \equiv \max_{\pp \in \cP(\pp_{0})} \E_{(\vx, y)\sim \pp} \ell(\vw; \vx, y) - \nu \chi^2(\pp_0, \pp).\]

  By \Cref{fact:firstOrderNecessary}, the first-order necessary and sufficient condition that corresponds to \(\pq_\vw := \argmax_{\pp \in \cP(\pp_{0})} L(\vw, \pp)\) is the following: for any $\pp \in \cP(\pp_0)$,
  \begin{equation}
  \label{eqn:FONC_pp}
    0 \ge \int {\grad_{\pp}L(\vw, \pq_\vw) \dd (\pp - \pq_\vw)}  = \int \grad_{\pp}L(\vw, \pq_\vw) \Bigl(\frac{\dd\pp}{\dd\pp_{0}} - \frac{\dd \pq_\vw}{\dd\pp_{0}} \Bigr) \dd \pp_{0},
  \end{equation}
  where we recall that both \(\grad_{\pp}L(\vw, \pq_\vw)\) and Radon–Nikodym derivatives $\frac{\dd\pp}{\dd\pp_{0}}, \frac{\dd \pq_\vw}{\dd\pp_{0}} $ are real-valued measurable functions on \(\R^{d} \times \R\).  We will also write \(\ell = \ell(\vw^{*}, \cdot, \cdot)\) for short. 

  We claim \Cref{eqn:FONC_pp} is satisfied 
  if there exists \(\xi \in \R\)
  and a bounded measurable function \(\psi \ge 0\) 
  such that
  \begin{equation}
    \label{eq:probability-simplex-optimality}
    \grad_{\pp}L(\vw, \pq_\vw) (\vx, y) =
    \begin{cases}
       \xi & \mbox{ if } \frac{\dd \pq_\vw}{\dd \pp_{0}} > 0\\
       \xi - \psi(\vx, y)  & \mbox{ otherwise.}
    \end{cases}
  \end{equation}
  Indeed, for any $\pp \in \cP(\pp_0) $,
  \begin{align*}
      & \int \grad_{\pp}L(\vw, \pq_\vw) \Bigl(\frac{\dd\pp}{\dd\pp_{0}} - \frac{\dd \pq_\vw}{\dd\pp_{0}} \Bigr) \dd \pp_{0} \\
      & = \int_{\frac{\dd \pq_\vw}{\dd \pp_{0}} > 0} \grad_{\pp}L(\vw, \pq_\vw) \Bigl(\frac{\dd\pp}{\dd\pp_{0}} - \frac{\dd \pq_\vw}{\dd\pp_{0}} \Bigr) \dd \pp_{0} + \int_{\frac{\dd \pq_\vw}{\dd \pp_{0}} = 0} \grad_{\pp}L(\vw, \pq_\vw) \frac{\dd\pp}{\dd\pp_{0}}\dd \pp_{0} \\
      & = \int_{\frac{\dd \pq_\vw}{\dd \pp_{0}} > 0} \xi \Bigl(\frac{\dd\pp}{\dd\pp_{0}} - \frac{\dd \pq_\vw}{\dd\pp_{0}} \Bigr) \dd \pp_{0} + \int_{\frac{\dd \pq_\vw}{\dd \pp_{0}} = 0} (\xi -\psi) \frac{\dd\pp}{\dd\pp_{0}}\dd \pp_{0} \\
      & \le \int_{\frac{\dd \pq_\vw}{\dd \pp_{0}} > 0} \xi \Bigl(\frac{\dd\pp}{\dd\pp_{0}} - \frac{\dd \pq_\vw}{\dd\pp_{0}} \Bigr) \dd \pp_{0} + \int_{\frac{\dd \pq_\vw}{\dd \pp_{0}} = 0}\xi \frac{\dd\pp}{\dd\pp_{0}}\dd \pp_{0}\\
      & =  \int_{\frac{\dd \pq_\vw}{\dd \pp_{0}} > 0} \xi \frac{\dd\pp}{\dd\pp_{0}} \dd \pp_{0} + \int_{\frac{\dd \pq_\vw}{\dd \pp_{0}} = 0}\xi \frac{\dd\pp}{\dd\pp_{0}}\dd \pp_{0} + \int_{\frac{\dd \pq_\vw}{\dd \pp_{0}} > 0} \xi \Bigl(- \frac{\dd \pq_\vw}{\dd\pp_{0}}\Bigr)  \dd\pp_{0}\\
      & = \int \xi \frac{\dd\pp}{\dd\pp_{0}}\dd \pp_{0}  + \int_{\frac{\dd \pq_\vw}{\dd \pp_{0}} > 0} \xi \Bigl(- \frac{\dd \pq_\vw}{\dd\pp_{0}} \Bigr) \dd \pp_{0} \stackrel{(i)} = \xi - \xi = 0,
  \end{align*}
  where (i) is because \(\int_{\frac{\dd \pq_\vw}{\dd \pp_{0}} > 0} \Bigl( \frac{\dd \pq_\vw}{\dd\pp_{0}} \Bigr) \dd\pp_{0} = \int \Bigl( \frac{\dd \pq_\vw}{\dd\pp_{0}} \Bigr) \dd\pp_{0} \).

  Observe from the definition of \(L(\vw, \pq_\vw)\) that  \(\grad_{\pp}L(\vw, \pq_\vw)(\vx, y) = \ell(\vw; \vx, y) - 2\nu (\frac{\dd \pq_\vw}{\dd \pp_{0}} (\vx, y)- 1)\).
  Plugging this into \Cref{eq:probability-simplex-optimality} and rearranging, we have,
  \begin{align*}
    \frac{\dd \pp^{*}}{\dd \pp_{0}}
    =
    \begin{cases}
      \frac{2 \nu + \ell - \xi}{2 \nu} & \mbox{if } \frac{\dd \pq_\vw}{\dd \pp_{0}} > 0 \\
      \frac{2 \nu + \ell - \xi + \psi}{2 \nu} &  \mbox{if } \frac{\dd \pq_\vw}{\dd \pp_{0}} = 0
    \end{cases}
  \end{align*}

  For the case where \( \frac{\dd \pp^{*}}{\dd \pp_{0}} > 0\), \(\frac{\dd \pp^{*}}{\dd \pp_{0}} = \frac{2 \nu + \ell - \xi}{2 \nu}\), so the condition \( \frac{\dd \pq_\vw}{\dd \pp_{0}} > 0\) becomes \({2 \nu + \ell - \xi} > 0 \). On the other hand, if the above condition fails, it has to be the case that \(\frac{\dd \pq_\vw}{\dd \pp_{0}} = 0\). Combining, we have
  \begin{align*}
    \frac{\dd \pq_\vw}{\dd \pp_{0}}
    = \begin{cases}
    \frac{2 \nu + \ell - \xi}{2 \nu} & \mbox{if } {2 \nu + \ell - \xi} > 0  \\
      0 &  \mbox{otherwise}
    \end{cases}
    = \frac{\max\{2 \nu + \ell - \xi, 0\}}{2 \nu}.
  \end{align*}
\end{proof}
Instead of using the expression in \Cref{lemma:key_fonc}, we will set $\nu$ to be big enough to ensure that there is no maximum in the expression for $\pq_\vw$. This is captured in \Cref{cor:key_fonc-no-max}.

\begin{corollary}[Simpler Closed-form $\pq_\vw$]
  \label{cor:key_fonc-no-max}
  Fix $\vw \in \R^d$. If \(\nu \ge \frac12 \E_{\pp_{0}}\ell(\vw)\), then
  \begin{align*}
    \frac{\dd \pq_\vw}{ \dd \pp_{0}}(\vx, y) &= 1 + \frac{\ell(\vw; \vx, y) - \E_{\pp_{0}}\ell(\vw)}{2\nu}.
  \end{align*}
  Similarly, if \(\nu \ge \frac12 \E_{\ep_{0}}\ell(\vw)\), then  \({\pq_\vw}^{(i)} > 0\) for all \(i \in [N]\), and
  \begin{align*}
     {\eq_\vw}^{(i)} = \ep_{0}^{(i)} + \frac{ \ell(\vw; \vx_{i}, y_{i}) - \E_{\ep_{0}}\ell(\vw) }{2\nu} \ep_{0}^{(i)}  \quad \mbox{for all } i \in [N].
  \end{align*}
  Furthermore, if \(\nu \ge \E_{\ep_{0}}\ell(\vw)\), then, in particular, for each coordinate \(j \in [N]\), we have
  \begin{equation*}
    \eq_\vw^{(j)} \ge \ep_{0}^{(j)} / 2
  \end{equation*}

  Similarly, if \(\nu \ge \E_{\pp_{0}}\ell(\vw)\), then for any non-negative function \(g\), we have
  \begin{align*}
    \int g \; \dd \pp_{\vw} \ge \frac12 \int g \; \dd \pp_{0}
  \end{align*}

  Recall from \Cref{def:loss_risk_opt} that when $\vw = \vw^*$, we define $\pp^* = \pq_{\vw^*}$ and $\ep^* = \eq_{\vw^*}$. If \(\nu \ge 8 \beta^2 \sqrt{6B}\sqrt{\OPTT + \epsilon}/{c_1}\) as assumed in \Cref{thm:main-formal}, then both conditions \(\nu \ge \E_{\ep_{0}}\ell(\vw)\) and \(\nu \ge \E_{\pp_{0}}\ell(\vw)\) hold.
\end{corollary}
\begin{proof}
Setting \(\nu \ge \frac12 \E_{\pp_{0}}\ell(\vw)\) and $\xi = \E_{\pp_{0}}\ell(\vw)$ in \Cref{lemma:key_fonc} implies \(\ell(\vw; \vx, y) - \E_{\pp_{0}}\ell(\vw; \vx, y) + 2\nu > 0 \), which, in turn, means
  \begin{align*}
    \frac{\dd \pq_\vw}{ \dd \pp_{0}}(\vx, y) &= \frac{\max\{\ell(\vw; \vx, y) - \E_{\pp_{0}}\ell(\vw; \vx, y) + 2\nu, 0\}}{2\nu} \\
    &= \frac {\ell(\vw; \vx, y) - \E_{\pp_{0}}\ell(\vw; \vx, y) + 2\nu}{2\nu}.
 \end{align*}
 The empirical version follows analogously.

 To establish the last claim, we show that  \(8 \beta^2 \sqrt{6B}\sqrt{\OPTThat}/{c_1} \ge \E_{\ep_{0}}\ell(\vw)\).
 By \Cref{cor:relationships-for-OPTs}, it holds that \[\sqrt{\OPTThat} \ge \OPThat = \E_{\ep^{*}}(\sigma(\vw^{*}\cdot\vx) - y)^{2} \ge \E_{\ep^{*}}(\sigma(\vw^{*}\cdot\vx) - y)^{2} - \nu \chi^{2}(\ep^{*}, \ep_{0}) = L(\vw^{*}, \ep^{*}).\] By definition of \(\ep^{*}\), we have \(L(\vw^{*}, \ep^{*}) \ge L(\vw^{*}, \ep_{0}) =  \E_{\ep_{0}}(\sigma(\vw^{*}\cdot\vx) - y)^{2} - \nu \chi^{2}(\ep_{0}, \ep_{0}) = \E_{\ep_{0}}(\sigma(\vw^{*}\cdot\vx) - y)^{2}\). Combining, we obtain \(8 \beta^2 \sqrt{6B}\sqrt{\OPTThat}/{c_1} \ge \E_{\ep_{0}}\ell(\vw)  8 \beta^2 \sqrt{6B}/c_{1}\). We conclude by observing \( 8 \beta^2 \sqrt{6B}/c_{1} \ge 1\).
\end{proof}

Another consequence of \Cref{cor:key_fonc-no-max} is a closed form expression for the risk, as a variance-regularized loss, similar to \cite{DuchiVariance19,Lam2013RobustSA}.

\begin{corollary}\label{cor:risk-distance-computation}
  Fix an arbitrary distribution \(\pp_{0}\). Recall the risk defined in \Cref{def:loss_risk_opt},
  \begin{align*}
    R(\vw; \pp_{0}) & := \max_{\pp \ll \pp_{0}}\E_{(\vx, y) \sim \pp} \ell(\vw;\vx, y) - \nu \chi^2(\pp, \pp_{0}).
  \end{align*}
  If \(\nu \ge \frac12 \E_{\ep_{0}}\ell(\vw)\), it holds that
  \begin{align*}
    \chi^2(\pq_{\vw}, \pp_{0}) &= \frac{\E_{\pp_{0}}[\ell^{2}(\vw)] - (\E_{\pp_{0}}[\ell(\vw)])^{2}}{4\nu^{2}} \\
    R(\vw; \pp_{0})  &= \E_{\pp_{0}} [\ell(\vw)] + \frac{\E_{\pp_{0}}[\ell^{2}(\vw)] - (\E_{\pp_{0}}[\ell(\vw)])^{2}}{4\nu}.
  \end{align*}
\end{corollary}

\begin{proof}
 Both of these follow from  \Cref{cor:key_fonc-no-max}. To see the first equality holds, observe that, 
 \begin{align*}
  \chi^2(\pq_{\vw}, \pp_{0}) &= \E_{\pp_0} \left( \frac{\dd \pq_\vw}{\dd \pp_0} - 1 \right)^2 \\
  &= \E_{\pp_0} \frac{(\ell(\vw; \vx, y) - \E_{\pp_0} \ell(\vw; \vx, y))^2}{4\nu^2} \\
  &= \frac{\E_{\pp_{0}}[\ell^{2}(\vw)] - (\E_{\pp_{0}}[\ell(\vw)])^{2}}{4\nu^{2}}.
 \end{align*}

The second equality follows by a similar substitution. 

 Setting $\dd \pq_\vw$ as per \Cref{cor:key_fonc-no-max}, we get
\begin{align*}
    R(\vw; \pp_{0}) 
    &= \E_{\pp_0} \left[ \left(\frac{\dd \pq_\vw}{\dd \pp_0}\right)\ell(\vw) \right] - \nu \chi^2(\pq_\vw, \pp_0)\\
    &= \E_{\pp_0} \left [\ell(\vw) \left(1 + \frac{\ell(\vw) - \E_{\pp_0} \ell(\vw)}{2\nu}\right)\right ] -       
        \frac{\E_{\pp_0}[\ell^2(\vw)] - (\E_{\pp_0}[\ell(\vw)])^2}{4\nu}\\
    &= \E_{\pp_0}[\ell(\vw)] +  \frac{ \E_{\pp_0} [\ell(\vw)^2] - (\E_{\pp_0} [\ell(\vw)])^2}{2\nu}  -
        \frac{\E_{\pp_0}[\ell^2(\vw)] - (\E_{\pp_0} [\ell(\vw)])^2}{4\nu}\\
    &= \E_{\pp_0} [\ell(\vw)] + \frac{\E_{\pp_0} [\ell^2(\vw)] - (\E_{\pp_0} [\ell(\vw)])^2}{4\nu}.
\end{align*}
  
\end{proof}

Finally, an important consequence of \Cref{lemma:key_fonc} is that it is possible to efficiently compute the risk of a given vector $\vw$ with respect to $\ep_0$. We use this to compare the risk of our final output with the risk that is achieved by the zero vector.

  \subsection{Concentration}
\label{subsec:concentration-ep}
The expression we get in \Cref{cor:key_fonc-no-max} for 
$\pp^*$ in terms of $\pp_0$ allows us to translate
concentration properties of $\pp_0$ to $\pp^*$. 
We first state and prove a helper lemma, \Cref{lemma:l-w-star-concentration}, that shows $\E_{\pp_0} \ell(\vw^*) \approx \E_{\ep_0} \ell(\vw^*)$.
Note that this is for the reference distribution $\pp_0$,
and not the target distribution $\pp^*$.

For ease of notation, we define $U := 2\beta^{2 }W^{2}(S^{2} + C_{M}^{2}B^{2} \log^{2}( W B \beta/ \epsilon))$, which is the upper bound for the loss value in \Cref{eq:loss-bound} throughout this section.

\begin{lemma}
  \label{lemma:l-w-star-concentration}
Suppose \(\pp_{0}\) satisfies \Cref{assump:boundedness}. Then for any fixed $\vw \in \cB(W)$ and all \(t > 0\), it holds that  
\[\abs{\E_{\pp_{0}}\ell(\vw) - \E_{\ep_{0}(N)}\ell(\vw)} \le t\] 
with probability at least 
\(1 - 2 \exp\Bigl(\frac{-t^{2}N}{8(\beta^{2 }W^{2}(S^{2} + C_{M}^{2}B^{2} \log^{2}( W B \beta/ \epsilon)))^2} \Bigr ).\) In particular,  the above inequality holds for \(\vw^{*}\).
\end{lemma}
\begin{proof}
  By \Cref{eq:loss-bound}, \(\forall i \in [N]\), \(0 \le (\sigma(\vw \cdot \vx_{i}) - y_{i})^{2} \le U \).
  Hoeffding's inequality (\Cref{lem:hoeffding}) now implies that, for all \(t > 0\),
  \begin{align*}
    \Pr\Bigl[\sum_{i = 1}^{N}{\frac1N\ell(\vw^{*}; \vx_{i}, y_{i})} - \E_{\pp_{0}}\ell(\vw^{*}) \ge t \Bigr] \le 2 \exp\Bigl(\frac{-t^{2}N}{2U^2} \Bigr ).
  \end{align*}
  Rearranging and plugging the definition of \(U\), we get the lemma.
\end{proof}

We now use \Cref{lemma:l-w-star-concentration} show that
bounded Lipschitz functions concentrate with respect to
$\pp^*$. 

\begin{lemma}
  \label{lemma:distributional-empirical}
  Fix \(\zeta > 0\). Let \(h = h(\vz; \vx, y): \cB(\zeta) \times \R^{d} \times \R \rightarrow \R\) be  a measurable function with respect to \(\vx, y\) that satisfies the condition that \(\abs{h(\vz;\cdot;\cdot)} \le b\) almost surely. Then, for
  \(N = O_{B, S, \beta}\Bigl(\frac{b^{2}}{t^{2}}\Bigl(1 + \frac{W^{4}\log^{4}(W/\epsilon)}{\nu^{2}}\Bigr) \log(1/\delta)\Bigr)\) samples drawn from the reference distribution \(\pp_{0}\) to construct \(\ep_{0}(N)\), for any fixed \(\vz \in \cB(\zeta)\),  with probability at least \(1- 4\delta\), it holds that
  \[\abs{\E_{(\vx, y)\sim \ep^{*}}[h(\vz; \vx, y)] - \E_{(\vx, y) \sim \pp^{*}}[h(\vz; \vx, y)]} \le t.\]
  Moreover, suppose \(\vz \mapsto h(\vz; \vx, y)\) is \(a\)-Lipschitz. Then, for
  \[N = O_{B, S, \beta}\Bigl(\frac{b^{2}}{t^{2}}\Bigl(1 + \frac{W^{4}\log^{4}(W/\epsilon)}{\nu^{2}}\Bigr)(d \log(\zeta a / t) + \log(1/\delta))\Bigr)\]
  with probability at least \(1-4\delta\), it holds that for all \(\vz \in \cB(\zeta)\),
  \[\abs{\E_{(\vx, y)\sim \ep^{*}}[h(\vz; \vx, y)] - \E_{(\vx, y) \sim \pp^{*}}[h(\vz; \vx, y)]} \le t.\]
\end{lemma}

\begin{proof}
  We use \Cref{lemma:key_fonc} to change the distribution with respect to which we are taking the expectation, \[\E_{\ep^{*}}[h(\vz)] = \E_{\ep_{0}}\Bigl[h(\vz; \vx, y)\frac{\ell(\vw^{*}; \vx, y) - \E_{\ep_{0}}\ell(\vw^{*}; \vx, y) + 2\nu}{2\nu}\Bigr].\]
  \Cref{lemma:l-w-star-concentration} now implies that with probability $1 -2\exp(-2Nt^2/(bU/2\nu)^2)$,
  \[\E_{\ep^{*}}[h(\vz)] = \E_{\ep_{0}}\Bigl[h(\vz; \vx, y)\frac{\ell(\vw^{*}; \vx, y) - \E_{\pp_{0}}\ell(\vw^{*}; \vx, y) + 2\nu}{2\nu}\Bigr] \pm  \frac t {4}.\]
  We now show that the expectation on the right hand side concentrates. To this end, we will use Hoeffding's inequality (\Cref{lem:hoeffding}). To apply this,  we will need a bound on the quantity in the expectation. 
  We bound this via an application of \Cref{eq:loss-bound} and the fact that $|h|\leq b$ to get,
  \[\abs[\Big]{h(\vz;\vx,y)\frac{\ell(\vw^{*}) - \E_{\pp_{0}}\ell(\vw^{*}) + 2\nu}{2\nu}} \le b \left( 1 +  \frac{U}{2\nu}\right). \]

  This means, with probability at least
  \(1  - 2\exp(-2t^{2} N /  (b^{2} (1 + U/2\nu)^2))\),
   \begin{equation}
     \label{eq:sharpness-hoeffding}
    \abs[\Big]{\E_{\pp_{0}}\Bigl[h(\vz)\frac{\ell(\vw^{*}) - \E_{\pp_{0}}\ell(\vw^{*}) + 2\nu}{2\nu}\Bigr] - E_{\ep_{0}}\Bigl[h(\vz)\frac{\ell(\vw^{*}) - \E_{\pp_{0}}\ell(\vw^{*}) + 2\nu}{2\nu}\Bigr]}  \le \frac{t}{2}.
  \end{equation}
  Since \(\vw \mapsto h(\vw)\) is \(a\)-Lipschitz, a standard net argument over \(\exp(\bigO{d \log(\zeta a / t )})\) vectors yields: with probability at least 
  \( 1 - 2\exp(\bigO{d \log(\zeta a / t ) - t^{2} N /(b^{2} (1 + U/2\nu)^2))}\)
  , it holds that for all \(\vz \in \cB(\zeta)\),
  \begin{equation}
     \label{eq:sharpness-hoeffding-all-w}
    \abs[\Big]{\E_{\pp_{0}}\Bigl[h(\vz)\frac{\ell(\vw^{*}) - \E_{\pp_{0}}\ell(\vw^{*}) + 2\nu}{2\nu}\Bigr] - E_{\ep_{0}}\Bigl[h(\vz)\frac{\ell(\vw^{*}) - \E_{\pp_{0}}\ell(\vw^{*}) + 2\nu}{2\nu}\Bigr] } \le t.
  \end{equation}

Putting things together, we see that if we choose 
\[N = \bigOmega[\Big]{\frac{b^{2}}{t^{2}}\Bigl(1 + \frac{U^{2}}{\nu^{2}}\Bigr)(d \log(\zeta a / t) + \log(1/\delta))},\]
with probability at least \(1 - 4\delta\), for all \(\vz \in \cB(\zeta)\),
  \[\abs{\E_{\ep^{*}}[h(\vz)] - \E_{\pp^{*}}[h(\vz)]} \le t.\]
\end{proof}

\subsection{Sharpness and Optimal Loss Value}
\label{subsec:emp_sharpness_loss}
Finally, as a consequence of \Cref{lemma:distributional-empirical}, we can derive that $\ep^*$ satisfies sharpness, $ \OPThat \approx \OPT$ and $\OPTThat \approx \OPTT$.

\begin{lemma}[Shaprness for \(\ep^{*}\)]
  \label{lemma:sharpness-empirical-appendix}
  Suppose \Cref{assump:margin,assump:concentration,assump:boundedness} are satisfied, then for large enough \(N\):
  \[N = \widetilde O_{B, S, \beta, \alpha, \gamma, \lambda}\Bigl(\frac{ W^{4} }{\eps^{2}}\Bigl(1 + \frac{W^{4}\log^{4}(1/\epsilon)}{\nu^{2}}\Bigr)(d  + \log(1/\delta))\Bigr),\]
  with  probability at least \(1 - 4 \delta\), for all \(\vw \in \cB(2\norm{\vw^{*}})\) with \(\norm{\vw - \vw^{*}} \ge \sqrt\epsilon\) and \(\vu \in \cB(1)\),
  \begin{align}
    \E_{\vx \sim \ep^{*}_{\vx}}[(\sigma(\vw \cdot \vx) - \sigma(\vw^{*} \cdot \vx))(\vw \cdot \vx - \vw^{*} \cdot \vx)]   & \ge  (c_{0}/2)\norm{\vw - \vw^{*}}_{2}^{2} \label{eq:sharpness-empirical-appendix} \\
    \E_{\vx \sim \ep^{*}_{\vx}}[(\vx \cdot \vu)^{\tau}] & \le 6 B  \label{eq:moment-bounds-empirical-appendix} \quad \mbox{for } \tau = 2, 4.
  \end{align}
\end{lemma}
\begin{proof}
\Cref{fact:sharpness} shows that $\pp^*$ the conditions above (with different constants).
We need to translate these to $\ep^*$. 
For each of the inequalities above, 
we will do this via an application of 
\Cref{lemma:distributional-empirical}.

\emph{Proof of \Cref{eq:sharpness-empirical-appendix}:}
   Set \(h\) in
  \Cref{lemma:distributional-empirical} to be \(h(\vw;\vx, y) := (\sigma(\vw \cdot \vx) - \sigma(\vw^{*} \cdot \vx))(\vw \cdot \vx - \vw^{*} \cdot \vx)\). We proceed to set the constants \(a\) and \(b\) that used in \Cref{lemma:distributional-empirical}.
  \Cref{eq:dot-product-bound} implies that for all \((\vx, y)\) in the support of \(\ep_{0}\), \(\abs{h(\vw;\vx,y)} \le 4 \beta W^{2} S^{2}  =: b\).
  Also, \(\vw \mapsto h(\vw)\) is
  \( a := 2 W S^{2}(\beta + 1)\sqrt{d}\)-Lipschitz as a consequence of
  \Cref{eq:x-norm-bound,eq:dot-product-bound}.

  \Cref{lemma:distributional-empirical} now gives us that for
  \(N = \widetilde O_{B, S, \beta}\Bigl(\frac{ W^{4} }{t^{2}}\Bigl(1 + \frac{W^{4}\log^{4}(1/\epsilon)}{\nu^{2}}\Bigr)(d  + \log(1/\delta))\Bigr)\) with probability at least \(1 - 4\delta\), for all \(\vw \in \cB(2\norm{\vw^{*}})\),
  \[\E_{\vx \sim \ep^{*}_{\vx}}[(\sigma(\vw \cdot \vx) - \sigma(\vw^{*} \cdot \vx))(\vw \cdot \vx - \vw^{*} \cdot \vx)]  \ge c_{0} \norm{\vw - \vw^{*}}_{2}^{2} - t.\]

  Using the fact that \(\norm{\vw - \vw^{*}}_{2} \ge \sqrt{\epsilon}\), we set \(t = c_{0}\epsilon / 2\), giving us the sample complexity
  \begin{align}\label{eq:sample-complexity-sharpness}
    N = \widetilde O_{B, S, \beta, \alpha, \gamma, \lambda}\Bigl(\frac{ W^{4} }{\eps^{2}}\Bigl(1 + \frac{W^{4}\log^{4}(1/\epsilon)}{\nu^{2}}\Bigr)(d  + \log(1/\delta))\Bigr).
  \end{align}

\emph{Proof of \Cref{eq:moment-bounds-empirical-appendix}:}
  This follows analogously to the proof above. Set \(h(\vu; \vx, y) = (\vx \cdot \vu)^{\tau}\) in \Cref{lemma:distributional-empirical} for \(\tau = 2, 4\) and we proceed to calculate constants \(a, b\). By \Cref{eq:dot-product-bound,eq:x-norm-bound}, it holds that \(h(\vu) \le S^{4} =: b\) and \(\vu \mapsto h(\vu)\) is \(a := 4 S^{4} \sqrt{d}\)-Lipschitz. Setting \(t = B\), by \Cref{lemma:distributional-empirical}, for \( N =  \widetilde O_{B, S, \beta}\Bigl(\Bigl(1 + \frac{W^{4}\log^{4}(W/\epsilon)}{\nu^{2}}\Bigr)(d  + \log(1/\delta))\Bigr)\) the conclusion follows.
  Note that this is dominated by \Cref{eq:sample-complexity-sharpness}.
\end{proof}

We now show that $\OPT \approx \widehat \OPT$. 

\begin{lemma}
  \label{lemma:opt-concentration}
Suppose \Cref{assump:margin,assump:concentration,assump:boundedness} are satisfied and the sample size \(N\) is large enough and  \(  N  =  \widetilde O_{B, S, \beta}\Bigl(\frac{W^{4}\log^{4}(1/\epsilon)}{t^{2}}\Bigl(1 + \frac{W^{4}\log^{4}(1/\epsilon)}{\nu^{2}}\Bigr) \log(1/\delta)\Bigr)\). Then for any fixed $\vw \in \cB(W)$ and all \(t > 0\), it holds that
\[\abs{\E_{\pp^{*}}\ell(\vw; \vx, y) - \E_{\ep^{*}}\ell(\vw; \vx, y)} \le t\]
with probability at least
\(1 - 4\delta.\) In particular, the above inequality holds for \(\vw^{*}\), i.e., \(\abs{\OPT - \OPThat} \le t\).
\end{lemma}

\begin{proof}
  Set \(b := \beta^{2 }W^{2}(2S^{2} + 2C_{M}^{2}B^{2} \log^{2}( W B \beta/ \epsilon)  \ge \norm{\ell(\vw)}_{2}\) in \Cref{lemma:distributional-empirical}, where the inequality is a consequence of \Cref{lem:bounded_samples}. Then, with
  \(N  =  \widetilde O_{B, S, \beta}\Bigl(\frac{W^{4}\log^{4}(1/\epsilon)}{t^{2}}\Bigl(1 + \frac{W^{4}\log^{4}(1/\epsilon)}{\nu^{2}}\Bigr) \log(1/\delta)\Bigr)\) samples, with probability $1-4\delta$, \(\abs{\E_{\pp^{*}}\ell(\vw; \vx, y) - \E_{\ep^{*}}\ell(\vw; \vx, y)} \le t\).
\end{proof}
Finally, we show $\OPTT \approx \OPTThat$. 
\begin{lemma}
  \label{lemma:opt2-concentration}
Suppose \Cref{assump:margin,assump:concentration,assump:boundedness} are satisfied and the sample size \(N\) is large enough and \( N  =  \widetilde O_{B, S, \beta}\Bigl(\frac{W^{8}\log^{8}(1/\epsilon)}{t^{2}}\Bigl(1 + \frac{W^{4}\log^{4}(1/\epsilon)}{\nu^{2}}\Bigr) \log(1/\delta)\Bigr)\). Then for any fixed $\vw \in \cB(W)$ and all \(t > 0\), it holds that
\[\abs{\E_{\pp^{*}}\ell^{2}(\vw; \vx, y) - \E_{\ep^{*}}\ell^{2}(\vw; \vx, y)} \le t\]
with probability at least
\(1 - 4\delta.\) In particular, the above inequality holds for \(\vw^{*}\), i.e., \(\abs{\OPTT - \OPTThat} \le t\).
\end{lemma}

\begin{proof}
Analogous to the previous proof, observe that \(\norm{\ell^{2}(\vw)}_{2} \le  8 \beta^{4}W^{4}(S^{4} + C_{M}^{4}B^{4} \log^{4}( W B \beta/ \epsilon) =: b\). By \Cref{lemma:distributional-empirical}, with
  \( N  =  \widetilde O_{B, S, \beta}\Bigl(\frac{W^{8}\log^{8}(1/\epsilon)}{t^{2}}\Bigl(1 + \frac{W^{4}\log^{4}(1/\epsilon)}{\nu^{2}}\Bigr) \log(1/\delta)\Bigr),\)
  with probability $1-4\delta$, it holds that \(\abs{\E_{\pp^{*}}\ell^{2}(\vw; \vx, y) - \E_{\ep^{*}}\ell^{2}(\vw; \vx, y)} \le t\).
\end{proof}

We capture properties of $\OPT$ and $\OPTT$ in the following corollary. 

\begin{corollary}[Properties of $\OPT, \OPTT$]\label{cor:relationships-for-OPTs}
    Suppose \Cref{assump:margin,assump:concentration,assump:boundedness} are satisfied and the sample size \(N\) is large enough and  \(  N  =  \widetilde O_{B, S, \beta}\Bigl(\frac{W^{8}\log^{8}(1/\epsilon)}{t^{2}}\Bigl(1 + \frac{W^{4}\log^{4}(1/\epsilon)}{\nu^{2}}\Bigr) \log(1/\delta)\Bigr)\). Then for all \(t > 0\), the following hold:
    \begin{enumerate}
        \item  With probability $1-4\delta$, \(\abs{\OPTT - \OPTThat} \le t\).
        \item  With probability $1-4\delta$, \(\abs{\OPT - \OPThat} \le t\).
        \item \(\OPT \leq \sqrt{\OPTT} \) and
        \(\OPThat \leq \sqrt{\OPTThat}\).
    \end{enumerate}
\end{corollary}
\begin{proof}
The first two items follow immediately from \Cref{lemma:opt2-concentration} and \Cref{lemma:opt-concentration}. The third item is a consequence of Cauchy-Schwarz. 
\end{proof}

\section{Gap Upper Bound}
\label{app:gap-upper-bound}
To prove \Cref{lemma:gap-upper-bound}, we need to construct an upper bound on the gap \(\Gap(\vw_i, \ep_i) = L(\vw_{i}, \ep^{*}) - L(\vw^{*}, \ep_{i})\). To achieve this, we establish an upper bound on \(L(\vw_{i}, \ep^{*})\),  which motivates the update rule for  \(\ep_{i}\). We also establish a lower bound on  \(L(\vw^{*}, \ep_{i})\), which guides the update rule for \(\vw_{i}\)  and the construction of \(\vg_{i}\). Note that the construction of the lower bound is more challenging here, due to the nonconvexity of the square loss. This is where most of the (non-standard) technical work happens. To simplify the notation, we use \(\phi(\ep) := \chi^{2}(\ep, \ep_{0})\) throughout this section.

\paragraph{Upper bound on \(L(\vw_{i}, \ep^{*})\).} We begin the analysis with the construction of the upper bound, which is used for defining the dual updates. Most of this construction follows a similar argument as used in other primal-dual methods such as \cite{song2021variance,diakonikolas2022fast}.

\begin{lemma}[Upper Bound on $a_i L(\vw_{i}, \ep^{*})$]\label{lemma:upper-bound-lagrangian}
Let \(\ep_{i}\) evolve as outlined in \Cref{line:p}. Then, 
for all $i \geq 1,$ 
\begin{align*}
 a_i L(\vw_{i}, \ep^{*}) \leq \; & a_i L(\vw_{i}, \ep_{i}) + (\nu_0 + \nu A_{i-1}) D_{\phi}(\ep^*, \ep_{i-1}) - (\nu_0 + \nu A_i) D_{\phi}(\ep^*, \ep_i)\\
 &-  (\nu_0 + \nu A_{i-1}) D_{\phi}(\ep_i, \ep_{i-1}). 
\end{align*}
\end{lemma}
\begin{proof}
  Recall that \(\phi(\ep) := \chi^{2}(\ep, \ep_{0})\). Observe that $L(\vw_{i}, \ep^{*}) $ as a function of $\ep^{*}$ is linear minus the nonlinearity $\nu \phi$. We could directly maximize this function and define $\ep_i$ correspondingly (which would lead to a valid upper bound); however, such an approach appears insufficient for obtaining our results. Instead, adding and subtracting $(\nu_0 + \nu A_{i-1}) D_{\phi}(\ep^*, \ep_{i-1})$, we have
    \begin{align}
        a_i L(\vw_{i}, \ep^{*}) &= a_i L(\vw_{i}, \ep^{*}) - (\nu_0 + \nu A_{i-1}) D_{\phi}(\ep^*, \ep_{i-1}) + (\nu_0 + \nu A_{i-1}) D_{\phi}(\ep^*, \ep_{i-1})\notag\\
        &= h(\ep^*) + (\nu_0 + \nu A_{i-1}) D_{\phi}(\ep^*, \ep_{i-1}),\label{eq:lagrangian-ub-1}
    \end{align}
  where we define, for notational convenience: 
  \[ h(\ep)  :=  a_{i}L(\vw_{i}, \ep) - (\nu_0 + {\nu A_{i-1}} )D_{\phi}(\ep, \ep_{i-1}).\]
  Observe that by the definition of $\ep_i,$ $h(\ep)$ is maximized by $\ep_i.$ Hence, using the definition of a Bregman divergence, we have that
    \begin{align*}
        h(\ep^*) &= h(\ep_i) + \innp{\nabla h(\ep_i), \ep^* - \ep_i} + D_h(\ep^*, \ep_i)\\
        &\leq h(\ep_i) - (\nu_0 + \nu A_i) D_{\phi}(\ep^*, \ep_i),
    \end{align*}
  where in the inequality we used that $\innp{\nabla h(\ep_i), \ep^* - \ep_i} \leq 0$ (as $\ep_i$ maximizes $h$) and $D_h(\ep^*, \ep_i) = - (\nu_0 + \nu A_i) D_{\phi}(\ep^*, \ep_i)$ (as $h(\ep)$ can be expressed as $- (\nu_0 + \nu A_{i})\phi(\ep)$ plus terms that are either linear in $\ep$ or independent of it. See \Cref{fact:bregman-linear-blind}). Combining with \Cref{eq:lagrangian-ub-1} and the definition of $h$ and simplifying, the claimed bound follows. %
\end{proof}

An important feature of \Cref{lemma:upper-bound-lagrangian} is that the first two Bregman divergence terms usefully telescope, while the last one is negative and can be used in controlling the error terms arising from the algorithmic choices.

\paragraph{Lower bound on $L(\vw^{*}, \ep_{i})$.}

The more technical part of our analysis concerns the construction of a lower bound on $L(\vw^{*}, \ep_{i})$, which leads to update rule for $\vw^*.$ In standard, Chambolle-Pock-style primal-dual algorithms \cite{chambolle2011first,alacaoglu2022complexity,song2021variance}, where the coupling $L(\vw, \ep)$ between the primal and the dual is \emph{bilinear}, the lower bound would be constructed using an analogue of the upper bound, with a small difference to correct for the fact that $\vw_i$ is updated before $\ep_i$ and so one cannot use information about $\ep_i$ in the $\vw_i$ update. This is done using an extrapolation idea, which replaces $\ep_i$ with an extrapolated value from prior two iterations and controls for the introduced error.

In our case, however, the coupling is not only nonlinear, but also \emph{nonconvex} because $\ell(\vw; \vx, y) = (\sigma(\vw \cdot \vx) - y)^2$ is nonconvex. Nonlinearity is an issue because if we were to follow an analogue of the construction from \Cref{lemma:upper-bound-lagrangian}, we would need to assume that we can efficiently minimize over $\vw$ the sum of $L(\vw, \ep)$ and a convex function (e.g., a quadratic), which translates into proximal point updates for the $L_2^2$ loss for which efficient computation is generally unclear. Nonlinearity alone (but assuming convexity) has been handled in the very recent prior work \cite{mehta2024primal}, where this issue is addressed using convexity of the nonlinear function to bound it below by its linear approximation around $\vw_i.$ Unfortunately, as mentioned before, this approach cannot apply here as we do not have convexity. Instead, we use a rather intricate argument that relies on monotonicity and Lipschitzness properties of the activation $\sigma$ and structural properties of the problem which only hold with respect to the target distribution $\ep^*$ (and the empirical target distribution $\pp^*$, due to our results from \Cref{lemma:sharpness-empirical}. Handling these issues related to nonconvexity of the loss in the construction of the upper bound is precisely what forces us to choose chi-square as the measure of divergence between distributions; see \Cref{lemma:main-auxiliary} and the discussion therein.

\begin{proposition}\label{prop:lower-bound-lagrangian}
    Consider the sequence \(\{\vw_{i}\}_{i}\) evolving as per \Cref{line:w}. Under the setting in which \Cref{lemma:sharpness-empirical} holds, %
    we have for all $i \geq 1,$
    \begin{align*}
        a_i L(\vw^{*}, \ep_{i}) \geq \;& L(\vw_{i}, \ep_{i}) -a_i E_i - (\nu_0 + \nu A_{i-2})D_{\phi}(\ep_{i-1}, \ep_{i-2})\\
        &+\frac{1 + 0.5 c_1 A_{i-1}}{2}\|\vw^* - \vw_i\|_2^2 - \frac{1 + 0.5 c_1 A_{i-1}}{2}\|\vw^* - \vw_{i-1}\|_2^2\\
    &+ \frac{1 + 0.5 c_1 A_{i-1}}{4}\|\vw_i - \vw_{i-1}\|_2^2 - \frac{1 + 0.5 c_1 A_{i-2}}{4}\|\vw_{i-1} - \vw_{i-2}\|_2^2\\
    &+a_i \innp{\E_{\ep_{i}}[\vv(\vw_{i}; \vx, y)] - \E_{\ep_{i-1}}[\vv(\vw_{i-1}; \vx, y)], \vw^* - \vw_i}\\
        &- a_{i-1}\innp{\E_{\ep_{i-1}}[\vv(\vw_{i-1}; \vx, y)] - \E_{\ep_{i-2}}[\vv(\vw_{i-2}; \vx, y)], \vw^* - \vw_{i-1}},
    \end{align*}
where is $E_i$ is defined by \Cref{eq:Ei-def}.
\end{proposition}
\begin{proof}
From the definition of $L(\vw^{*}, \ep_{i}),$ we have:
\begin{align*}
    \quad L(\vw^{*}, \ep_{i}) = \E_{\ep_{i}}[(\sigma(\vw^{*} \cdot \vx) - y)^2] -  \nu D(\ep_{i}, \ep_{0}).
  \end{align*}    
Writing $(\sigma(\vw^{*} \cdot \vx) - y)^2 = ((\sigma(\vw^{*} \cdot \vx) - \sigma(\vw_i \cdot \vx)) + (\sigma(\vw_i \cdot \vx) - y))^2$ and expanding the square, we have
\begin{align}
    L(\vw^{*}, \ep_{i})
    =\; &  \E_{\ep_{i}}[(\sigma(\vw_{i} \cdot \vx) - y)^2] -  \nu D(\ep_{i}, \ep_{0}) + \E_{\ep_{i}}[(\sigma(\vw^{*} \cdot \vx) - \sigma(\vw_{i} \cdot \vx))^2] \notag\\
    &+ \E_{\ep_{i}}[2(\sigma(\vw_{i} \cdot \vx) - y)(\sigma(\vw^{*} \cdot \vx) - \sigma(\vw_{i} \cdot \vx))] \notag\\
     =\; &  L(\vw_{i}, \ep_{i})%
     + S_i, \label{eq:lb-lagrangian-1}
  \end{align}    
where for notational convenience we define 
\begin{equation}\label{eq:Si-def}
    S_i := \E_{\ep_{i}}[(\sigma(\vw^{*} \cdot \vx) - \sigma(\vw_{i} \cdot \vx))^2] + \E_{\ep_{i}}[2(\sigma(\vw_{i} \cdot \vx) - y)(\sigma(\vw^{*} \cdot \vx) - \sigma(\vw_{i} \cdot \vx))].
\end{equation}
Observe that $L(\vw_{i}, \ep_{i})$ on the right-hand side also appears in the upper bound on $ L(\vw_i, \ep^*)$ in \Cref{lemma:upper-bound-lagrangian} and so it will get cancelled out when $L(\vw^{*}, \ep_{i})$ is subtracted from $ L(\vw_i, \ep^*)$ in the gap computation. Thus, we only need to focus on bounding $S_i$. 
This requires a rather technical argument, which we defer to \Cref{lemma:main-auxiliary} below. Instead, we call on \Cref{lemma:main-auxiliary} to state that
\begin{equation}\label{eq:Si-lb}
    S_i \geq \E_{\ep_{i}}[\innp{\vv(\vw_i; \vx, y), \vw^* - \vw_i}] - E_i
\end{equation}
and carry out the rest of the proof under this assumption (which is proved in \Cref{lemma:main-auxiliary}). 

At this point, we have obtained a ``linearization'' that was needed to continue by mimicking the construction of the upper bound. However, $\vv(\vw_i; \vx, y)$ depends on $\vw_i,$ and so trying to define $\vw_i$ based on this quantity would lead to an implicitly defined update, which is generally not efficiently computable. Instead, here we use the idea of extrapolation: instead of defining a step w.r.t.\ $\vw_i,$ we replace $\E_{\ep_{i}}[\vv(\vw_i; \vx, y)]$ by an ``extrapolated gradient'' defined by (cf.\ \Cref{line:g} in \Cref{alg:main}):
\[\vg_{i-1} = \E_{\ep_{i-1}}[\vv(\vw_{i-1}; \vx, y)] + \frac{a_{i-1}}{a_i}(\E_{\ep_{i-1}}[\vv(\vw_{i-1}; \vx, y)] - \E_{\ep_{i-2}}[\vv(\vw_{i-2}; \vx, y)]).\]
Combining with the bound on $S_i$ from \Cref{eq:Si-lb} and simplifying, we now have
\begin{equation}\label{eq:aiSi-lb}
\begin{aligned}
    a_i S_i \geq\; & a_i \innp{\vg_{i-1}, \vw^* - \vw_i} - a_i E_i\\
    &+ a_i\innp{\E_{\ep_{i}}[\vv(\vw_{i}; \vx, y) - \vg_{i-1}], \vw^* - \vw_{i}}.
\end{aligned} 
\end{equation}
Let $\psi(\vw) = a_i \innp{\vg_{i-1}, \vw } + \frac{1 + 0.5 c_1 A_{i-1}}{2}\|\vw - \vw_{i-1}\|_2^2$ and observe that (by \Cref{line:w} in \Cref{alg:main}) $\vw_{i} = \argmin_{\vw \in \cB(W)}\psi(\vw)$. Then, by a similar argument as in the proof of \Cref{lemma:upper-bound-lagrangian}, since $\psi$ is minimized by $\vw_i$ and is a quadratic function in $\vw_i$, we have
\begin{equation}\label{eq:3PI-wi}
\begin{aligned}
     a_i \innp{\vg_{i-1}, \vw^* - \vw_i} \geq\; & \frac{1 + 0.5 c_1 A_{i-1}}{2}\|\vw^* - \vw_i\|_2^2 - \frac{1 + 0.5 c_1 A_{i-1}}{2}\|\vw^* - \vw_{i-1}\|_2^2\\
     &+ \frac{1 + 0.5 c_1 A_{i-1}}{2}\|\vw_i - \vw_{i-1}\|_2^2.
\end{aligned}
\end{equation}
On the other hand, by the definition of $\vg_i,$ we have
\begin{equation}\label{eq:wi-telescope+err}
    \begin{aligned}
      &\;  a_i \innp{\E_{\ep_{i}}[\vv(\vw_{i}; \vx, y)] - \vg_{i-1}, \vw^* - \vw_{i}}\\
      =\; 
        & a_i \innp{\E_{\ep_{i}}[\vv(\vw_{i}; \vx, y)] - \E_{\ep_{i-1}}[\vv(\vw_{i-1}; \vx, y)], \vw^* - \vw_i}\\
        &- a_{i-1}\innp{\E_{\ep_{i-1}}[\vv(\vw_{i-1}; \vx, y)] - \E_{\ep_{i-2}}[\vv(\vw_{i-2}; \vx, y)], \vw^* - \vw_{i-1}}\\
        &+ a_{i-1}\innp{\E_{\ep_{i-1}}[\vv(\vw_{i-1}; \vx, y)] - \E_{\ep_{i-2}}[\vv(\vw_{i-2}; \vx, y)], \vw_i - \vw_{i-1}}
    \end{aligned}
\end{equation}
The first two terms on the right-hand side of \Cref{eq:wi-telescope+err} telescope, so we focus on bounding the last term. We do so using Young's inequality (\Cref{fact:young}) followed by $\kappa$-Lipschitzness of $\vv,$ which leads to 
\begin{align}
  &\; - a_{i-1}\innp{\E_{\ep_{i-1}}[\vv(\vw_{i-1}; \vx, y)] - \E_{\ep_{i-2}}[\vv(\vw_{i-2}; \vx, y)], \vw_i - \vw_{i-1}}\notag \\
  \leq \;& \frac{{a_{i-1}}^2}{1 + 0.5 c_1 A_{i-1}}\|\E_{\ep_{i-1}}[\vv(\vw_{i-1}; \vx, y)] - \E_{\ep_{i-2}}[\vv(\vw_{i-2}; \vx, y)]\|_2^2 + \frac{1 + 0.5 c_1 A_{i-1}}{4}\|\vw_i - \vw_{i-1}\|_2^2\notag\\
  \overset{(i)}{\leq}\;& \frac{2{a_{i-1}}^2\kappa^2}{1 + 0.5 c_1 A_{i-1}}\|\vw_{i-1} - \vw_{i-2}\|_2^2 + \frac{2{a_{i-1}}^2G^2}{1 + 0.5 c_1 A_{i-1}} D_\phi(\ep_{i-1}, \ep_{i-2}) + \frac{1 + 0.5 c_1 A_{i-1}}{4}\|\vw_i - \vw_{i-1}\|_2^2\notag\\
  \overset{(ii)}{\leq}\;& \frac{1 + 0.5 c_1 A_{i-2}}{4}\|\vw_{i-1} - \vw_{i-2}\|_2^2 + \frac{1 + 0.5 c_1 A_{i-1}}{4}\|\vw_i - \vw_{i-1}\|_2^2 + (\nu_0 + \nu A_{i-2})D_{\phi}(\ep_{i-1}, \ep_{i-2}), \label{eq:wi-err}
\end{align}
where in ($ii$) we used $\frac{2{a_{i-1}}^2\kappa^2}{1+0.5 c_1 A_{i-1}} \leq \frac{1+0.5 c_1 A_{i-2}}{4}$ and $\frac{2{a_{i-1}}^2G^2}{1+0.5 c_1 A_{i-1}} \leq \nu_0 + \nu A_{i-2},$ which both hold by the choice of the step size,
while ($i$) follows by boundedness and $\kappa$-Lipschitzness of $\vv$ and \Cref{lemma:boundedness-to-norm}, using
\begin{align}
  &\; \| \E_{\ep_{i-1}}[\vv(\vw_{i-1}; \vx, y)] - \E_{\ep_{i-2}}[\vv(\vw_{i-2}; \vx, y)]\|_2^2\notag\\
  =\; &  \| \E_{\ep_{i-1}}[\vv(\vw_{i-1}; \vx, y)] -\E_{\ep_{i-2}}[\vv(\vw_{i-1}; \vx, y)] + \E_{\ep_{i-2}}[\vv(\vw_{i-1}; \vx, y)]  - \E_{\ep_{i-2}}[\vv(\vw_{i-2}; \vx, y)]\|_2^2\notag\\
   \leq\; & 2\|\E_{\ep_{i-1}}[\vv(\vw_{i-1}; \vx, y)] -\E_{\ep_{i-2}}[\vv(\vw_{i-1}; \vx, y)]\|_2^2 + 2\|\E_{\ep_{i-2}}[\vv(\vw_{i-1}; \vx, y)]  - \E_{\ep_{i-2}}[\vv(\vw_{i-2}; \vx, y)]\|_2^2\notag\\
   \leq\; & 2G^2 D_{\phi}(\ep_{i-1}, \ep_{i-2}) + 2\kappa^2 \|\vw_{i-1} - \vw_{i-2}\|_2^2.\label{eq:pivi-pi-1vi-1}
\end{align}

Combining \Cref{eq:aiSi-lb,eq:3PI-wi,eq:wi-telescope+err,eq:wi-err}, we now have
\begin{align*}
    a_i S_i \geq \; & -a_i E_i +\frac{1+0.5 c_1 A_{i-1}}{2}\|\vw^* - \vw_i\|_2^2 - \frac{1+0.5 c_1 A_{i-1}}{2}\|\vw^* - \vw_{i-1}\|_2^2\\
    &- (\nu_0 + \nu A_{i-2}) D_{\phi}(\ep_{i-1}, \ep_{i-2})\\
    &+ \frac{1+0.5 c_1 A_{i-1}}{4}\|\vw_i - \vw_{i-1}\|_2^2 - \frac{1+0.5 c_1 A_{i-2}}{4}\|\vw_{i-1} - \vw_{i-2}\|_2^2\\
    &+a_i \innp{\E_{\ep_{i}}[\vv(\vw_{i}; \vx, y)] - \E_{\ep_{i}}[\vv(\vw_{i-1}; \vx, y)], \vw^* - \vw_i}\\
        &- a_{i-1}\innp{\E_{\ep_{i-1}}[\vv(\vw_{i-1}; \vx, y)] - \E_{\ep_{i-1}}[\vv(\vw_{i-2}; \vx, y)], \vw^* - \vw_{i-1}}.
\end{align*}
To complete the proof, it remains to combine the last inequality with \Cref{eq:lb-lagrangian-1}. 
\end{proof}

\begin{lemma}\label{lemma:main-auxiliary}
    Let $S_i$ be defined by \Cref{eq:Si-def}. Then, under the setting of \Cref{prop:lower-bound-lagrangian}, we have
    \begin{align*}
        S_i \geq \E_{\ep_{i}}[\innp{\vv(\vw; \vx, y), \vw^* - \vw_i}] - E_i,
    \end{align*}
    where
    \begin{equation}\label{eq:Ei-def}
        E_i = \frac{c_1}{4}\norm{\vw^* - \vw_i}_2^2 + \frac{8\beta^2 \sqrt{6B}\sqrt{\OPTThat}}{ c_1}\chi^2(\ep_i, \ep^*) + \frac{48 \beta^2 B \OPThat}{c_1}.%
    \end{equation}
\end{lemma}

\begin{proof}
    Define the event \(\cG  = \{(\vx, y): \sigma(\vw_{i} \cdot \vx) - y \ge 0\}\). Then,
  \begin{align*}
    &\; \E_{\ep_{i}}[2(\sigma(\vw_{i} \cdot \vx) - y)(\sigma(\vw^{*} \cdot \vx) - \sigma(\vw_{i} \cdot \vx))] \\
     = \; &\E_{\ep_{i}}[2(\sigma(\vw_{i} \cdot \vx) - y)(\sigma(\vw^{*} \cdot \vx) - \sigma(\vw_{i} \cdot \vx)) \Ind_{\cG} + 2(\sigma(\vw_{i} \cdot \vx) - y)(\sigma(\vw^{*} \cdot \vx) - \sigma(\vw_{i} \cdot \vx)) \Ind_{\cG^{c}}] \\
     \ge\; &\E_{\ep_{i}}[ \Ind_{\cG}2(\sigma(\vw_{i} \cdot \vx) - y) \sigma'(\vw_{i}\cdot \vx) (\vw^{*}\cdot \vx - \vw_{i} \cdot \vx)] \\
    &  + \E_{\ep_{i}}[ \Ind_{\cG^{c}}2(\sigma(\vw_{i} \cdot \vx) - y) \sigma'(\vw^{*}\cdot \vx) (\vw^{*}\cdot \vx - \vw_{i}\cdot\vx)],
  \end{align*}
  where the last inequality uses convexity of \(\sigma(\cdot)\) to bound the term that involves $\Ind_{\cG}$ and concavity of $-\sigma(\cdot)$ to bound the term that involves $\Ind_{\cG^{c}}$ and where $\sigma'$ denotes any subderivative of $\sigma$ (guaranteed to exist, due to convexity and Lipschitzness).

  Recall that $\vv(\vw_i; \vx, y) = 2\beta(\sigma(\vw_{i} \cdot \vx) - y) \vx $. Using that $\sigma'(t) = \beta + (\sigma'(t) - \beta)$ for all $t$ and combining with the inequality above, we see
  \begin{align}
    &\; \E_{\ep_{i}}[2(\sigma(\vw_{i} \cdot \vx) - y)(\sigma(\vw^{*} \cdot \vx) - \sigma(\vw_{i} \cdot \vx))] \notag\\
     \ge\; &\E_{\ep_{i}}[\innp{\vv(\vw_i; \vx, y), \vw^* - \vw_i}]\notag \\
    &  + 2\E_{\ep_{i}}[ \Ind_{\cG}(\sigma(\vw_{i} \cdot \vx) - y) (\sigma'(\vw_i\cdot \vx) - \beta) (\vw^{*}\cdot \vx - \vw_{i}\cdot\vx)]\\
    & + 2\E_{\ep_{i}}[ \Ind_{\cG^{c}}(\sigma(\vw_{i} \cdot \vx) - y) (\sigma'(\vw^{*}\cdot \vx) - \beta) (\vw^{*}\cdot \vx - \vw_{i}\cdot\vx)], \label{eq:vdef-1}
  \end{align}
  and so to prove the lemma we only need to focus on bounding the terms in the last two lines.

  Recall that $\sigma$ is assumed to be monotonically increasing and $\beta$-Lipschitz, and so $0 \leq \sigma'(\vw^{*}\cdot \vx) \leq \beta.$ Thus, we have 
  \begin{align}
      &\; (\sigma(\vw_{i} \cdot \vx) - y) (\sigma'(\vw^{*}\cdot \vx) - \beta) (\vw^{*}\cdot \vx - \vw_{i}\cdot\vx)\notag\\
      =\; & (\sigma(\vw^{*} \cdot \vx) - y) (\sigma'(\vw^{*}\cdot \vx) - \beta) (\vw^{*}\cdot \vx - \vw_{i}\cdot\vx)\notag\\
      &+ (\sigma(\vw_{i} \cdot \vx) - \sigma(\vw^{*} \cdot \vx)) (\sigma'(\vw^{*}\cdot \vx) - \beta)  (\vw^{*}\cdot \vx - \vw_{i}\cdot\vx)\notag\\
      \geq\;& (\sigma(\vw^{*} \cdot \vx) - y) (\sigma'(\vw^{*}\cdot \vx) - \beta) (\vw^{*}\cdot \vx - \vw_{i}\cdot\vx), \label{eq:vdef-2}
  \end{align}
  where we have used $\sigma'(\vw^{*}\cdot \vx) - \beta \leq 0$ (by Lipschitzness) and $(\sigma(\vw_{i} \cdot \vx) - \sigma(\vw^{*} \cdot \vx))  (\vw^{*}\cdot \vx - \vw_{i}\cdot\vx) \leq 0$ (by monotonicity of $\sigma$). By the same argument, 
  \begin{align*}
       &\; (\sigma(\vw_{i} \cdot \vx) - y) (\sigma'(\vw_i\cdot \vx) - \beta) (\vw^{*}\cdot \vx - \vw_{i}\cdot\vx)\notag\\
       \geq\;& (\sigma(\vw^{*} \cdot \vx) - y) (\sigma'(\vw_i\cdot \vx) - \beta) (\vw^{*}\cdot \vx - \vw_{i}\cdot\vx).
  \end{align*}

  To complete the proof of the lemma, it remains to bound the expectation of the term in \Cref{eq:vdef-2}. We proceed using that $|\sigma'(\vw\cdot \vx) - \beta| \leq \beta$, $\forall \vw$, and thus for $\vw \in \{\vw^*, \vw_i\}$:
  \begin{align*}
      &\; |(\sigma(\vw^{*} \cdot \vx) - y) (\sigma'(\vw\cdot \vx) - \beta) (\vw^{*}\cdot \vx - \vw_{i}\cdot\vx)|\\
      \leq\;& \beta |\sigma(\vw^{*} \cdot \vx) - y||\vw^{*}\cdot \vx - \vw_{i}\cdot\vx|.
  \end{align*}
  Taking the expectation on both sides, and combining with  \Cref{eq:vdef-2}, we further have
  \begin{align}
  &\; - 2\E_{\ep_{i}}[ \Ind_{\cG}(\sigma(\vw_{i} \cdot \vx) - y) (\sigma'(\vw_i\cdot \vx) - \beta) (\vw^{*}\cdot \vx - \vw_{i}\cdot\vx)]\notag\\
      &\; -2\E_{\ep_{i}}[ \Ind_{\cG^{c}}(\sigma(\vw_{i} \cdot \vx) - y) (\sigma'(\vw^{*}\cdot \vx) - \beta) (\vw^{*}\cdot \vx - \vw_{i}\cdot\vx)]\notag \\
      \le \;& 2 \beta\, \E_{\ep_{i}}[ (\Ind_{\cG}+\Ind_{\cG^{c}}) |\sigma(\vw^{*} \cdot \vx) - y||\vw^{*}\cdot \vx - \vw_{i}\cdot\vx| ]\notag\\
      = \;& 2 \beta\, \E_{\ep_{i}}[|\sigma(\vw^{*} \cdot \vx) - y||\vw^{*}\cdot \vx - \vw_{i}\cdot\vx| ]\notag\\
      =\;& 2 \beta \int |\sigma(\vw^{*} \cdot \vx) - y||\vw^{*}\cdot \vx - \vw_{i}\cdot\vx| \dd \ep_i\notag\\
      =\;& 2 \beta \int |\sigma(\vw^{*} \cdot \vx) - y||\vw^{*}\cdot \vx - \vw_{i}\cdot\vx| \dd \ep^*\notag\\
      &+ 2 \beta \int |\sigma(\vw^{*} \cdot \vx) - y||\vw^{*}\cdot \vx - \vw_{i}\cdot\vx| (\dd \ep_i - \dd \ep^*). \label{eq:int-vdef-joint}
  \end{align}
  In the last equality, the first integral is just the expectation with respect to $\ep^*,$ and thus using Cauchy-Schwarz inequality, the definition of $\OPT$, and \Cref{lemma:sharpness-empirical-appendix}, the first term in \Cref{eq:int-vdef-joint} can be bounded by
  \begin{align}
      &\; \int |\sigma(\vw^{*} \cdot \vx) - y||\vw^{*}\cdot \vx - \vw_{i}\cdot\vx| \dd \ep^*\notag\\
      \leq\; & \sqrt{\E_{\ep^*}[(\sigma(\vw^{*} \cdot \vx) - y)^2] \E_{\ep^*}[(\vw^{*}\cdot \vx - \vw_{i}\cdot\vx)^2]}\notag\\
      \leq \; & \sqrt{\OPThat} \sqrt{6B} \|\vw^* - \vw_i\|_2\notag\\
      \leq \;& \frac{24 \beta B \OPThat }{c_1} + \frac{ c_1}{16\beta}\|\vw^* - \vw_i\|_2^2, \label{eq:first-int-bnd-aux}
  \end{align}
    where the last line is by Young's inequality and the second last line uses \Cref{lemma:sharpness-empirical}.
  
  For the remaining integral in \Cref{eq:int-vdef-joint}, using the definition of chi-square divergence and Cauchy-Schwarz inequality, we have
  \begin{align}
      &\; \int |\sigma(\vw^{*} \cdot \vx) - y||\vw^{*}\cdot \vx - \vw_{i}\cdot\vx| (\dd \ep_i - \dd \ep^*)\notag\\
      =\; & \int |\sigma(\vw^{*} \cdot \vx) - y||\vw^{*}\cdot \vx - \vw_{i}\cdot\vx| \frac{(\dd \ep_i - \dd \ep^*)}{\sqrt{\dd\ep^*}}\sqrt{\dd\ep^*}\notag\\
      \overset{(i)}{\leq} \;& \sqrt{\chi^2(\ep_i, \ep^*)\E_{\ep^*}[(\sigma(\vw^{*} \cdot \vx) - y)^2(\vw^{*}\cdot \vx - \vw_{i}\cdot\vx)^2]}\notag\\
      \overset{(ii)}{\le} \; & \chi^2(\ep_i, \ep^*)^{1/2}\E_{\ep^*}[(\sigma(\vw^{*} \cdot \vx) - y)^4]^{1/4} \E_{\ep^*}[(\vw^{*}\cdot \vx - \vw_{i}\cdot\vx)^4]^{1/4} \notag \\
  \overset{(iii)}{\le}\; & \chi^2(\ep_i, \ep^*)^{1/2}\OPTThat^{1/4} (6 B \norm{\vw^{*} - \vw_{i}}_2^{4})^{1/4}\notag\\
  \overset{(iv)}{\le}\; & \frac{4\beta \sqrt{6B}\sqrt{\OPTThat}}{c_1}\chi^2(\ep_i, \ep^*) + \frac{c_1}{16\beta} \norm{\vw^{*} - \vw_{i}}_2^{2}, \label{eq:aux-2nd-int}
  \end{align}
  where ($i$) is by Cauchy-Schwarz, ($ii$) is by Cauchy-Schwarz again, ($iii$) is by the definition of $\OPTThat$ and \Cref{lemma:sharpness-empirical-appendix}, and (iv) is by Young's inequality. 

  To complete the proof, it remains to plug \Cref{eq:int-vdef-joint,eq:first-int-bnd-aux,eq:aux-2nd-int} back into \Cref{eq:vdef-1}, and simplify.
\end{proof}

\paragraph{Gap upper bound proof of \Cref{lemma:gap-upper-bound}.} Combining the upper and lower bounds from \Cref{lemma:upper-bound-lagrangian} and \Cref{prop:lower-bound-lagrangian}, we are now ready to prove \Cref{lemma:gap-upper-bound}, which we restate below.

\lemgapub*
\begin{proof}
Combining the upper bound on $a_i L(\vw_{i}, \ep^{*})$ from \Cref{lemma:upper-bound-lagrangian} with the lower bound on $a_i L(\vw^{*}, \ep_{i})$ from \Cref{prop:lower-bound-lagrangian} and recalling that $\Gap(\vw_i, \ep_i) = L(\vw_{i}, \ep^{*}) - L(\vw^{*}, \ep_{i})$ and $A_i = A_{i-1} + a_i$, we have
\begin{align*}
    a_i \Gap(\vw_i, \ep_i) \leq &\; a_i E_i \\
    &+ (\nu_0 + \nu A_{i-1}) D_{\phi}(\ep^*, \ep_{i-1}) - (\nu_0 +\nu A_i) D_{\phi}(\ep^*, \ep_i)\\
    &+ (\nu_0 + \nu A_{i-2})D_{\phi}(\ep_{i-1}, \ep_{i-2}) - (\nu_0 + \nu A_{i-1})D_{\phi}(\ep_{i}, \ep_{i-1})\\
    &+ \frac{1 +  0.5 c_1 A_{i-1}}{2}\|\vw^* - \vw_{i-1}\|_2^2 - \frac{1 +  0.5 c_1 A_{i}}{2}\|\vw^* - \vw_i\|_2^2\\
    &+ \frac{1 +  0.5 c_1 A_{i-2}}{4}\|\vw_{i-1} - \vw_{i-2}\|_2^2 - \frac{1 +  0.5 c_1 A_{i-1}}{4}\|\vw_i - \vw_{i-1}\|_2^2\\
    &+ a_{i-1}\innp{\E_{\ep_{i-1}}[\vv(\vw_{i-1}; \vx, y)] - \E_{\ep_{i-2}}[\vv(\vw_{i-2}; \vx, y)], \vw^* - \vw_{i-1}}\\
    &- a_i \innp{\E_{\ep_{i}}[\vv(\vw_{i}; \vx, y)] - \E_{\ep_{i-1}}[\vv(\vw_{i-1}; \vx, y)], \vw^* - \vw_i}. 
\end{align*}
Observe that except for the first line on the right-hand side of the above inequality, all remaining terms telescope. Summing over $i = 1, 2, \dots, k$ and recalling that, by convention, $a_0 = A_0 = a_{-1} = A_{-1} = 0$, $\vw_{-1} = \vw_0$, and $\ep_{-1} = \ep_0,$ we have
\begin{equation}\label{eq:gap-telescoped}
    \begin{aligned}
        \sum_{i=1}^k a_i \Gap(\vw_i, \ep_i) \leq \;& \sum_{i=1}^k a_i E_{i}+ \frac{1}{2}\|\vw^* - \vw_0\|_2^2 + \nu_0 D_\phi(\ep^*, \ep_0)\\
        &- \frac{1 +  0.5 c_{1}A_{k}}{2}\|\vw^* - \vw_k\|_2^2 - (\nu_0 + A_k)D_\phi(\ep^*, \ep_k)\\
        &-a_k \innp{\E_{\ep_{k}}[\vv(\vw_{k}; \vx, y)] - \E_{\ep_{k-1}}[\vv(\vw_{k-1}; \vx, y)], \vw^* - \vw_k}\\
        & - \frac{1 +  0.5 c_{1} A_{k-1}}{4}\|\vw_k - \vw_{k-1}\|_2^2 - (\nu_0 + \nu A_{k-1})D_\phi(\ep_k, \ep_{k-1}).
    \end{aligned}
\end{equation}
To complete the proof, it remains to bound $a_k |\innp{\E_{\ep_{k}}[\vv(\vw_{k}; \vx, y)] - \E_{\ep_{k-1}}[\vv(\vw_{k-1}; \vx, y)], \vw^* - \vw_k}|,$ which is done similarly as in the proof of \Cref{prop:lower-bound-lagrangian}. In particular,
\begin{align}
    &\; a_k |\innp{\E_{\ep_{k}}[\vv(\vw_{k}; \vx, y)] - \E_{\ep_{k-1}}[\vv(\vw_{k-1}; \vx, y)], \vw^* - \vw_k}|\notag\\
    \overset{(i)}{\leq} \;& \frac{{a_k}^2}{1 +  0.5 c_{1} A_k}\|\E_{\ep_{k}}[\vv(\vw_{k}; \vx, y)] - \E_{\ep_{k-1}}[\vv(\vw_{k-1}; \vx, y)]\|_2^2 + \frac{1 +  0.5 c_{1} A_k}{4}\|\vw^* - \vw_k\|_2^2\notag\\
    \overset{(ii)}{\leq}\; & \frac{{a_k}^2}{1 +  0.5 c_{1} A_k}\Big(2G^2 D_{\phi}(\ep_{k}, \ep_{k-1}) + 2\kappa^2 \|\vw_{k} - \vw_{k-1}\|_2^2\Big) + \frac{1 +  0.5 c_{1} A_k}{4}\|\vw^* - \vw_k\|_2^2\notag\\
    \overset{(iii)}{\leq}\; & (\nu_0 + \nu A_{k-1}) D_{\phi}(\ep_{k}, \ep_{k-1}) + \frac{1 +  0.5 c_{1} A_{k-1}}{4}\|\vw_{k} - \vw_{k-1}\|_2^2 + \frac{1 +  0.5 c_{1} A_k}{4}\|\vw^* - \vw_k\|_2^2, \label{eq:final-inprod}
\end{align}
where ($i$) is by Young's inequality and ($ii$) is by \Cref{eq:pivi-pi-1vi-1}, and ($iii$) is by $\frac{2G^2{a_k}^2}{1 +  0.5 c_{1} A_k} \leq \nu_0 + \nu A_{k-1}$ and $\frac{2\kappa^2{a_k}^2}{1 +  0.5 c_{1} A_k} \leq \frac{1 +  0.5 c_{1} A_{k-1}}{4}$, which both hold by the choice of the step sizes in \Cref{alg:main}.

To complete the proof, it remains to plug \Cref{eq:final-inprod} back into \Cref{eq:gap-telescoped}, use the definition of $E_i$ from \Cref{eq:Ei-def}, and simplify.
\end{proof}

\section{Omitted Proofs in Main Theorem}\label{sec:bounded-iterate}
\boundediterate*
\begin{proof}[Proof of \Cref{claim:bounded-iterate}]
    It trivially holds that \(\vzero = \vw_{0} \in \cB(2\norm{\vw^{*}}_{2})\). Suppose \(\norm{\vw_{i}}_{2} \le 2\norm{\vw^{*}}_{2}\) for all iterations \(i \le t \) where \(t \ge 0\). Then 
\begin{align*}
      & \; -\frac{12\beta^2 {B}}{c_1}\OPThat A_{k} +  \sum_{i=1}^k a_{i}\frac{c_1}{2} \|\vw_{i} - \vw^{*}\|_{2}^{2} +  \sum_{i=1}^k \nu a_{i}D_{\phi}(\ep^*, \ep_{i})  + a_{k+1} \Gap(\vw_{k+1}, \ep_{k+1})\\
      \le & \sum_{i=1}^{k+1} a_i \Gap(\vw_i, \ep_i)\\
        \leq & \frac{1}{2}\|\vw^* - \vw_0\|_2^2 + \nu_0 D_\phi(\ep^*, \ep_0)
         - \frac{1 + 0.5 c_1A_{k+1}}{2}\|\vw^* - \vw_{k+1}\|_2^2 - (\nu_0 + \nu A_{k+1})D_\phi(\ep^*, \ep_{k+1})\\
         &+ \sum_{i=1}^{k+1} a_i \frac{c_1}{4} \|\vw^* - \vw_i\|_2^2 + \frac{8\beta^2 \sqrt{6B}\sqrt{\OPTThat}}{c_1}\sum_{i=1}^{k+1} a_i \chi^2(\ep_i, \ep^*) + \frac{48 \beta^2 B\OPThat A_{k+1}}{c_1},
    \end{align*}
    where we used the gap upper bound \Cref{lemma:gap-upper-bound} again as it does not require $\vw \in \cB(\norm{\vw^*}_2)$. We proceed to deduce a different lower bound for $\Gap(\vw_{k+1}, \ep_{k+1})$. Similar to \Cref{thm:gap-lower-bound}, we break into two terms
    $L(\vw^{*}, \ep_{k+1}) - (-L(\vw^{*}, \ep^{*})) \ge \nu D_\phi(\ep^*, \ep_{k+1})$ and 
    $L(\vw_{k+1}, \ep^{*}) - L(\vw^{*}, \ep^{*}) = \mathbb{E}_{\ep^{*}}[(\sigma(\vw_{k+1} \cdot \vx) - y)^{2} - (\sigma(\vw^{*} \cdot \vx) - y)^{2}] \ge -\OPThat$, where the first term is bounded the same way as in \Cref{thm:gap-lower-bound}. Hence, $\Gap(\vw_{k+1}, \ep_{k+1}) \ge -\OPThat$. Therefore, we simplify as before
    \begin{align*}
      &  \frac{1 + 0.5 c_1A_{k+1}}{2}\|\vw^* - \vw_{k+1}\|_2^2 + (\nu_0 + \nu A_{k+1})D_\phi(\ep^*, \ep_{k+1}) - a_i \frac{c_1}{4} \|\vw^* - \vw_{k+1}\|_2^2  \\
       & \leq  \frac{1}{2}\|\vw^* - \vw_0\|_2^2 + \nu_0 D_\phi(\ep^*, \ep_0) + \frac{12\beta^2 {B}}{c_1}\OPThat A_{k}   + a_{k+1} \OPThat + \frac{48 \beta^2 B\OPThat A_{k+1}}{c_1},
    \end{align*}
    which implies by nonnegativity of Bregman divergence that:
    \begin{align*}
       \frac{2 + c_1A_{k}}{4}\|\vw^* - \vw_{k+1}\|_2^2  \leq  \frac{1}{2}\|\vw^* - \vw_0\|_2^2 + \nu_0 D_\phi(\ep^*, \ep_0) + \Bigl(\frac{60\beta^2 {B}}{c_1} A_{k}   + a_{k+1} \Bigl(1 + \frac{48 \beta^2 B} {c_1}\Bigr)\Bigr) \OPThat.
    \end{align*}
    Rearranging and using $2 + c_1A_{k} \ge 2$,
    \begin{align*}
       \|\vw^* - \vw_{k+1}\|_2^2  \leq  \|\vw^* - \vw_0\|_2^2 + 2 \nu_0 D_\phi(\ep^*, \ep_0) + \Bigl(\frac{240\beta^2 {B}}{c_1^2} + \frac{4 a_{k+1}}{2 + c_1 A_k} \Bigl(1 + \frac{48 \beta^2 B} {c_1}\Bigr)\Bigr) \OPThat,
    \end{align*}
    Our choice of stepsizes $a_i$ implies $\frac{a_{k+1}}{2 + c_1 A_k} \le 1 / \max\{\kappa, G\} \le 1$, hence 
    \begin{align*}
       \|\vw^* - \vw_{k+1}\|_2^2  \leq  \|\vw^* - \vw_0\|_2^2 + 2 \nu_0 D_\phi(\ep^*, \ep_0) + \Bigl(\frac{288\beta^2 {B}}{c_1^2} + \frac{1}{\max\{\kappa, G\}}\Bigr) \OPThat,
    \end{align*}

\begin{claim}\label{claim:ambiguity-radius}
    For $\nu \ge 8 \beta^2 \sqrt{6B}\sqrt{\OPTThat}/{c_1}$, it holds that \[\chi^2 (\ep^*, \ep_0) = \frac{\Var_{\ep_0}(\ell(\vw^*))}{4 \nu^2} \le \frac{\OPTThat}{2 \nu^2} \le c_1/(1536 \beta^4 B).\] 
    Similarly, for $\nu \ge \E_{\pp_{0}}\ell(\vw^{*}), 8 \beta^2 \sqrt{6B}\sqrt{\OPTT}/{c_1}$, it holds that \[\chi^2 (\pp^*, \pp_0) = \frac{\Var_{\pp_0}(\ell(\vw^*))}{4 \nu^2} \le \frac{\OPTT}{2 \nu^2} \le c_1/(1536 \beta^4 B).\]
\end{claim}
\begin{proof}[Proof of \Cref{claim:ambiguity-radius}]
    By \Cref{cor:risk-distance-computation}, $\chi^2 (\ep^*, \ep_0) = \frac{\E_{\ep_0}[\ell^2(\vw^*)] - (\E_{\ep_0}[\ell(\vw^*)])^2}{4 \nu^2} \le \frac{\E_{\ep_0}[\ell^2(\vw^*)]}{4 \nu^2} \le \frac{\OPTThat}{2 \nu^2}  \le c_1/(1536 \beta^4 B)$, where the second last inequality uses $\ep^* \ge \ep_0 / 2$ from \Cref{cor:key_fonc-no-max} and the last inequality comes from lower bound on $\nu$ in the assumption.

    The population version follows analogously. 
\end{proof}
    Since $D_\phi(\ep^*, \ep_0) = \chi^2 (\ep^*, \ep_0)$, by choosing $\nu_0 = 768 \beta^4 B \epsilon/ c_1$, we ensure $2 \nu_0 D_\phi(\ep^*, \ep_0) \le \epsilon$.
    
    By choosing $\nu_0$ small enough and initialization $\vw_0 = \vzero$, it holds that 
    \[\|\vw^* - \vw_{k+1}\|_2^2  \leq  \|\vw^*\|_2^2 + \epsilon + \Bigl(\frac{288\beta^2 {B}}{c_1^2} + \frac{1}{\max\{\kappa, G\}}\Bigr) \OPThat.\]

    We may assume without loss of generality that $\frac{1}{\max\{\kappa, G\}} \ll \frac{288\beta^2 {B}}{c_1^2}$ because both $\kappa$ and $G$ is $O(d)$ but the right hand side is an absolute constant. We may also assume without loss of generality that $\frac{300\beta^2 {B}}{c_1^2} \OPThat + \eps \le 2 \norm{\vw^*}^2_2$, thus completing the induction step $\|\vw^* - \vw_{k+1}\|_2 \le 2 \norm{\vw^*}_2$. The reason for the last no loss of generality is the following: otherwise, we can compare, per \Cref{claim:zero-tester}, the empirical risk of the output from our algorithm and of $\hat \vw = \vzero$ and output the solution with the lower risk to get an $\bigO{\OPT} + \eps$ solution. 
\end{proof}
\begin{claim}[Zero-Tester]
\label{claim:zero-tester}
In the setting of \Cref{thm:main-formal}, it is possible to efficiently check if $R(\vzero; \ep_0) > R(\widehat \vw; \ep_0)$ or not; where $\widehat \vw$ is the output of \Cref{alg:main}. 
\end{claim}
\begin{proof}
    Observe that $L(\vw, \ep) = \sum_{i=1}^N \ep_i (\sigma(\vw \cdot \vx) -y )^2 - \nu \chi^2(\ep, \ep_0)$ is $1/\nu$-strongly concave in $(\ep^{(1)}, \dots, \ep^{(N)})$. 
    Now, since $R(\vw; \ep_0) = \max_{\ep} L(\vw, \ep)$, we can estimate the risk at any given $\vw$ using standard maximization techniques (such as gradient descent). 

    To test which risk is larger, we  estimate $R(\vzero; \ep_0)$ and $R(\widehat \vw; \ep_0)$ to a necessary accuracy and then compare.

\end{proof}

\section{Parameter Estimation to Loss and Risk Approximation}
\label{app:parameter_estimation}
\Cref{thm:main-formal} shows that \Cref{alg:main} recovers a vector $\widehat \vw$ such that $\| \widehat \vw - \vw^* \|_2 \leq \sqrt{\OPT} + \sqrt{\eps}$, where $\OPT := \E_{\pp^*} (\sigma(\vw^* \cdot \vx) - y)^2$.

\subsection{Loss Approximation}

In this section we show that this implies that the neuron we recover achieves a constant factor approximation to the optimal squared loss. 
\begin{lemma}
Let $\pp^*$ satisfy \Cref{assump:margin} and \Cref{assump:concentration}. Suppose $(\widehat \vw, \ep)$ is the solution returned by \Cref{alg:main} when given $N = $ samples drawn from $\pp_0$. Then, 
 $\E_{\pp^*} (\sigma(\widehat\vw \cdot \vx) - y)^2 \leq O_{\beta, B}(\OPT) + \eps$.
\end{lemma}
\begin{proof}
Recall that $\sigma$ is $\beta$-Lipschitz, and \Cref{fact:sharpness} gives us $\E_{\pp^*} (\vu \cdot \vx)^2 \leq 5B$ for all unit vectors $\vu$. These together imply,
\begin{align*}
\E_{\pp^*} (\sigma(\widehat\vw \cdot \vx) - y)^2 &\leq 2\E_{\pp^*} (\sigma(\vw^* \cdot \vx) - y)^2 + 2 \E_{\pp^*} (\sigma(\vw^* \cdot \vx) - \sigma(\widehat\vw \cdot \vx) )^2\\
&\leq 2\OPT + 2\beta^2 \| \widehat \vw - \vw^* \|^2_2 ~\E_{\pp^*} \left( \frac{\widehat \vw - \vw^*}{\| \widehat \vw - \vw^* \|} \cdot \vx \right)^2\\
&\leq 2 \OPT + 2 \beta^2 (2 C_3^2 \OPT + 2 \eps) ~ 5B\\
&\leq (2 + 20 B \beta^2 C_3^2)  ~\OPT + {10 \beta^2  B} \eps.
\end{align*}

\end{proof}

  \subsection{Risk Approximation}
  Fix \(\hat\vw\) as output by \Cref{alg:main} and \(\vw^{*}\) as defined in \Cref{def:loss_risk_opt}.
  Since we are bounding the population risk throughout this subsection, we write \(R(\vw)  = R(\vw; \pp_{0})\) in short. The goal of this subsection is to show \[R(\hat\vw) - R(\vw^{*}) \le O(\OPT) + \epsilon.\]

We first introduce some convex analysis results that we rely on in this subsection:
  \begin{fact}[Strong convexity of chi-square divergence]\label{fact:chi-square-strong-convexity}
    Consider the space \(\cP(\pp_{0}) = \{\pp: \pp \ll \pp_{0}\}\). For \(\pp \in \cP(\pp_{0})\), we denote by \(\frac{\dd \pp}{\dd \pp_{0}}\) the Radon–Nikodym derivative of \(\pp\) with respect to \(\pp_{0}\), and we define \(\norm{\pp}_{\pp_{0}}^{*} = \sqrt{\int (\frac{\dd \pp}{\dd \pp_{0}})^{2}\dd \pp_{0} }\). Then \(\pp \mapsto \chi^{2}(\pp, \pp_{0})\) is 2-strongly convex with respect to \(\norm{\cdot}_{\pp_{0}}^{*}\).
  \end{fact}

  \begin{fact}\label{fact:dual-norm-prob-space}
    Consider the space  \(\cP(\pp_{0}) = \{\pp: \pp \ll \pp_{0}\}\). Denote by \(\innp{\cdot, \cdot}_{\pp_{0}}\) the inner product \(\innp{\ell_{1}, \ell_{2}}_{\pp_{0}} = \int \ell_{1} \ell_{2} \dd \pp_{0}\) and denote by \(\norm{\cdot}_{\pp_{0}}\) the correponding norm. Then \(\norm{\cdot}_{\pp_{0}}\) is the dual norm of \(\norm{\cdot}_{\pp_{0}}^{*}\) defined in \Cref{fact:chi-square-strong-convexity}.
  \end{fact}

\begin{definition}[Convex conjugate]
Given a convex function defined on a vector space \(\mathbb{E}\) denoted by \( f: \mathbb{E} \to \mathbb{R} \), its convex conjugate is defined as:
\[
f^*(y) = \sup_{x \in \mathbb{E}} \left( \langle y, x \rangle - f(x) \right)
\]
for all \( y \in \mathbb{E}^{*} \) where \(\mathbb{E}^{*}\) is the dual space of \(\mathbb{E}\) and \( \langle y, x \rangle \) denotes the inner product.

\end{definition}

\begin{fact}[{Conjugate Correspondence Theorem,~\cite[Theorem 5.26]{beck2017first}}]\label{fact:conjugate-correspondence}
Let \(\nu > 0\). If \( f : \mathbb{E} \to \mathbb{R} \) is a  \(\nu\)-strongly convex continuous function, then its convex conjugate \( f^* : \mathbb{E}^* \to \mathbb{R} \) is \(\frac{1}{\nu}\)-smooth.
\end{fact}

We are then able to state and prove the key technical corollary in this subsection:
  \begin{corollary}\label{thm:risk-smoothness}
    For any \(\pp_{0}\)-measurable function \(\ell: \R^{d} \times \R \to \R\),  let \(\cR(\ell) = \max_{\pp \ll \pp_{0}}\E_{\pp_{0}}\ell - \nu \chi^{2}(\pp, \pp_{0}) \). The function \(\cR(\cdot)\) is \(1/(2 \nu)\)-smooth with respect to the norm \(\norm{\cdot}_{\pp_{0}}\) defined in \Cref{fact:dual-norm-prob-space}.
  \end{corollary}

  \begin{proof}
    Observe by definition of the convex conjugate that \(\cR(\cdot)\) is the convex conjugate of the function \(\nu \chi^{2}(\cdot, \pp_{0})\). Since the function \(\nu \chi^{2}(\cdot, \pp_{0})\) is \(2\nu\)-strongly convex with respect to the norm \(\norm{\cdot}_{\pp_{0}}^{*}\) by \Cref{fact:chi-square-strong-convexity}, it follows from \Cref{fact:conjugate-correspondence} that \(\cR(\cdot)\) is \(1/(2\nu)\)-smooth with respect to the norm \(\norm{\cdot}_{\pp_{0}}\).
  \end{proof}

  For ease of presentation, we define the following quantities:
  let \(\ell^{*}(\vx, y) = (\sigma(\vw^{*} \cdot \vx) - y)^{2}\) and \(\hat\ell(\vx, y) = (\sigma(\hat\vw \cdot \vx) - y)^{2}\). We first compute \(\grad_{\ell}\cR(\ell^{*})\) by conjugate subgradient theorem.

  \begin{fact}[{Conjugate Subgradient Theorem~\cite[
Theorem 4.20]{beck2017first}}]\label{fact:conjugate-subgradient}
    Let \( f : \mathbb{E} \to \mathbb{R} \) be  convex and continuous. The following claims are equivalent for any \(x \in \mathbb{E}\) and \(y \in \mathbb{E}^{*}\):
    \begin{enumerate}
      \item \(\innp{y, x} = f(x) + f^{*}(y)\)
      \item \(y \in \partial f(x)\)
      \item \(x \in \partial f^{*}(y)\)
    \end{enumerate}
  \end{fact}
  \begin{corollary}\label{thm:risk-star-subgradient}
    Let \(\pp^{*}\) be as defined in \Cref{def:loss_risk_opt}. Then  \(\pp^{*} \in \partial_{\ell}\cR(\ell^{*})\).
  \end{corollary}
  \begin{proof}
    We verify that \(\cR(\ell^{*}) = \max_{\pp \ll \pp_{0}}\E_{\pp_{0}}[\sigma(\vw^{*} \cdot \vx  - y )^{2}] - \nu \chi^{2}(\pp, \pp_{0}) = \E_{\pp^{*}}[\sigma(\vw^{*} \cdot \vx  - y )^{2}] - \nu \chi^{2}(\pp^{*}, \pp_{0}) = \E_{\pp^{*}}\ell^{*} - \nu \chi^{2}(\pp^{*}, \pp_{0}) = \innp{\pp^{*}, \ell^{*}} - \nu \chi^{2}(\pp^{*}, \pp_{0}) \), where the second equality is the definition of \(\pp^{*}\) and the third equality is the definition of \(\ell^{*}\). By \Cref{fact:conjugate-subgradient} and observing that \(\cR(\cdot)\) is the convex conjugate of the function \(\nu \chi^{2}(\cdot, \pp_{0})\), we have \(\pp^{*} \in \partial_{\ell}\cR(\ell^{*})\).
  \end{proof}
  \begin{theorem}\label{thm:risk-bound}
    Suppose \Cref{cor:key_fonc-no-max} holds for both \(\vw^{*}\) and
    \(\hat\vw\) with respect to the population distribution.
    Then \[R(\hat\vw;\pp_{0}) - R(\vw^{*}; \pp_{0}) \le C_{4} (\OPT + \epsilon),\]
    where
    \(C_{4} = 1 + 2 (10 B \beta^{2} + c_{1})C_{3} + c_{1} \sqrt{5B} \beta^{2} C_{3}^{2} \).
    In particular, \Cref{cor:key_fonc-no-max} holds for both \(\vw^{*}\) and
    \(\hat\vw\) is satisfied under the assumptions in \Cref{thm:main-formal}.
  \end{theorem}

  \begin{proof}

    By the definition of smoothness, it holds that for any \( \pp \in \partial_{\ell}\cR(\ell^{*})\), we have
    \begin{align*}
      & \;R(\hat\vw;\pp_{0}) - R(\vw^{*}; \pp_{0}) = \cR(\hat\ell) - \cR(\ell^{*}) \le \innp{\pp, \hat\ell - \ell^{*}} + \frac{1}{2\nu} \norm{\hat\ell - \ell^{*}}^{2}_{\pp_{0}} 
    \end{align*}

    Hence it follows from \Cref{thm:risk-star-subgradient} that
    \begin{align*}
      & \;R(\hat\vw;\pp_{0}) - R(\vw^{*}; \pp_{0}) \le \innp{\pp^{*}, \hat\ell - \ell^{*}} + \frac{1}{2\nu} \norm{\hat\ell - \ell^{*}}^{2}_{\pp_{0}} \\
      = & \; \E_{\pp^{*}}[(\sigma(\hat\vw \cdot \vx) - y)^{2} - (\sigma(\vw^{*}\cdot \vx) - y)^{2}] + \frac{1}{2\nu} \E_{\pp_{0}}[((\sigma(\hat\vw \cdot \vx) - y)^{2} - (\sigma(\vw^{*}\cdot \vx) - y)^{2})^{2}]. 
    \end{align*}

   We use the following shorthand for ease of presentation.
  \begin{align*}
    \Psi(\vx, y) & = \sigma(\vw^{*}\cdot \vx) - y, \\
    \Delta(\vx, y) &= \sigma(\hat\vw \cdot \vx) - \sigma(\vw^{*}\cdot \vx).\\
  \end{align*}

  Then
  \begin{align*}
      & \;R(\hat\vw;\pp_{0}) - R(\vw^{*}; \pp_{0}) \le \innp{\pp^{*}, \hat\ell - \ell^{*}} + \frac{1}{2\nu} \norm{\hat\ell - \ell^{*}}^{2}_{\pp_{0}} \\
    = & \; \E_{\pp^{*}}[(\Delta + \Psi)^{2} - \Psi^{2}] + \frac{1}{2\nu} \E_{\pp_{0}}[((\Delta + \Psi)^{2} - \Psi^{2})^{2}] \\
    = & \; \E_{\pp^{*}}[\Delta^{2} + 2\Delta\Psi] + \frac{1}{2\nu} \E_{\pp_{0}}[\Delta^{2}(\Delta + 2\Psi)^{2}] \\
    \le & \;\E_{\pp^{*}}[\Delta^{2} + 2\Delta\Psi] +  \frac{1}{\nu} \E_{\pp_{0}}[\Delta^{2}(\Delta^{2} + 4\Psi^{2})],
  \end{align*}
  where the last inequality is the standard inequality \((a+b)^{2} \le 2a^{2} + 2b^{2}\).

  Recall in \Cref{fact:sharpness} that \(\E_{\pp^{*}}[\Delta \Psi] \ge c_{0} \norm{\hat\vw - \vw^{*}}_{2}^{2}\) and \(\E_{\pp^{*}}[ (\vx \cdot (\hat\vw -  \vw^{*}))^{\tau}]  \le 5B  \norm{\hat\vw - \vw^{*}}_{2}^{\tau}\).

  From the second- and fourth-moment bounds, we have \(\E_{\pp^{*}}[\Delta^{\tau}] = \E_{\pp^{*}}[(\sigma(\vx \cdot \hat\vw) - \sigma(\vx \cdot \vw^{*}))^{\tau}] \le \beta^{\tau} \E_{\pp^{*}}[ (\vx \cdot (\hat\vw -  \vw^{*}))^{\tau}]  \le 5B \beta^{\tau}  \norm{\hat\vw - \vw^{*}}_{2}^{\tau}\) for \(\tau = 2, 4\), where the second last inequality follows from \(\beta\)-Lipschitzness of \(\sigma(\cdot)\). Taking \(\tau = 2\) gives us a bound for \(\E_{\pp^{*}}[\Delta^{2}]\).

  For \(\E_{\pp^{*}}[\Delta \Psi]\), it follows from Cauchy-Schwarz that \(\E_{\pp^{*}}[\Delta \Psi] \le \sqrt{\E_{\pp^{*}}[\Delta^{2}] \E_{\pp^{*}}[\Psi^{2}] }  \le \sqrt{5 B \beta^{2} \OPT} \norm{\hat\vw - \vw^{*}}_{2}\).

  By \Cref{cor:key_fonc-no-max}, we have \(\E_{\pp_{0}}[\Delta^{4}] \le 2 \E_{\pp^{*}}[\Delta^{4}] \le 5 B \beta^{4} \norm{\hat\vw - \vw^{*}}^{4}_{2}\).

  Finally, similarly by \Cref{cor:key_fonc-no-max}, it follows additionally from Cauchy-Schwarz that  \(\E_{\pp_{0}}[\Delta^{2}\Psi^{2}] \le 2  \E_{\pp^{*}}[\Delta^{2}\Psi^{2}] \le 2 \sqrt{ \E_{\pp_{0}}[\Delta^{4}] \E_{\pp^{*}}[\Psi^{4}] } \le 2 \norm{\hat\vw - \vw^{*}}_{2}^{2} \sqrt{5 B \beta^{4} \OPTT}\). By \Cref{thm:main-formal}, we have \(\nu \ge 8 \beta^{2}\sqrt{6B} \sqrt{\OPTT + \epsilon} / c_{1}\) by assumption, hence \(4 \E_{\pp_{0}}[\Delta^{2}\Psi^{2}] / \nu \le c_{1} \norm{\hat\vw - \vw^{*}}_{2}^{2} \).
  
  Combining the above four bounds and the guarantee of \Cref{thm:main-formal} that \(\norm{\hat\vw - \vw^{*}}_{2}^{2} \le 2 C_{3}\OPT + 2{\epsilon}\), we have
   \begin{align*}
     & \;R(\hat\vw;\pp_{0}) - R(\vw^{*}; \pp_{0}) \\
     \le & \; 5B \beta^{2}  \norm{\hat\vw - \vw^{*}}_{2}^{2} + 2 \sqrt{5 B \beta^{2} \OPT} \norm{\hat\vw - \vw^{*}}_{2} +  c_{1} \norm{\hat\vw - \vw^{*}}_{2}^{2} + 5 B \beta^{4} \norm{\hat\vw - \vw^{*}}^{4}_{2} / \nu \\
     \le & \; \OPT + (10 B \beta^{2} + c_{1} + 5 B \beta^{4}  \norm{\hat\vw - \vw^{*}}_{2}^{2} / \nu)  \norm{\hat\vw - \vw^{*}}_{2}^{2} \\
     \le &\; \OPT + 2 (10 B \beta^{2} + c_{1})(C_{3}\OPT + \epsilon) + 40 B \beta^{4}  (C_{3}^{2} \OPT^{2} + \epsilon^{2}) / \nu.
   \end{align*}

   By \Cref{cor:relationships-for-OPTs}, \(\OPT \le \sqrt{\OPTT}\), hence \(\nu \ge 8 \beta^{2}\sqrt{6B} \sqrt{\OPTT + \epsilon} / c_{1} \ge 8 \beta^{2}\sqrt{6B} \max\{\OPT, \sqrt\epsilon\} / c_{1} \), hence
  \begin{align*}
     & \;R(\hat\vw;\pp_{0}) - R(\vw^{*}; \pp_{0}) \\
     \le &\; \OPT + 2 (10 B \beta^{2} + c_{1})(C_{3}\OPT + \epsilon) + c_{1} \sqrt{5B} \beta^{2}  (C_{3}^{2} \OPT^{2} + \epsilon^{2}) / \max\{\OPT, \sqrt\epsilon\} \\
    \le & \; \OPT + 2 (10 B \beta^{2} + c_{1})(C_{3}\OPT + \epsilon) + c_{1} \sqrt{5B} \beta^{2}  (C_{3}^{2} \OPT + \epsilon^{1.5}) \\
    = & \; (1 + 2 (10 B \beta^{2} + c_{1})C_{3} +  c_{1} \sqrt{5B} \beta^{2} C_{3}^{2} ) \OPT +  (2 (10 B \beta^{2} + c_{1}) + c_{1} \sqrt{5B} \beta^{2}) \epsilon.
   \end{align*}

  \end{proof}

\if@preprint
\else
\newpage
\section*{NeurIPS Paper Checklist}

\begin{enumerate}

\item {\bf Claims}
    \item[] Question: Do the main claims made in the abstract and introduction accurately reflect the paper's contributions and scope?
    \item[] Answer: \answerYes{} %
    \item[] Justification: The abstract and introduction clearly outline the open problem we provide the first results for, clearly state the assumptions involved (cf.\ \Cref{sec:problem-setup} --- Problem Setup and \Cref{sec:overview} --- Main Result), and give a detailed technical overview of our algorithm (cf.\ \Cref{sec:techniques} --- Technical Overview).

\item {\bf Limitations}
    \item[] Question: Does the paper discuss the limitations of the work performed by the authors?
    \item[] Answer: \answerYes{} %
    \item[] Justification: The limitations are clearly stated in the statements of each theorem and are discussed in the introduction of the paper and description of the key lemma (\Cref{lemma:main-auxiliary-main-body}).

\item {\bf Theory Assumptions and Proofs}
    \item[] Question: For each theoretical result, does the paper provide the full set of assumptions and a complete (and correct) proof?
    \item[] Answer: \answerYes{} %
    \item[] Justification: Each theorem statement provides all the assumptions and we provide complete proofs for all statements that are either in the main body of the paper or in the appendix.

    \item {\bf Experimental Result Reproducibility}
    \item[] Question: Does the paper fully disclose all the information needed to reproduce the main experimental results of the paper to the extent that it affects the main claims and/or conclusions of the paper (regardless of whether the code and data are provided or not)?
    \item[] Answer: \answerNA{} %
    \item[] Justification: The paper is theoretical in nature and does not include experiments.

\item {\bf Open access to data and code}
    \item[] Question: Does the paper provide open access to the data and code, with sufficient instructions to faithfully reproduce the main experimental results, as described in supplemental material?
    \item[] Answer: \answerNA{} %
    \item[] Justification: The paper is theoretical in nature and does not include experiments.

\item {\bf Experimental Setting/Details}
    \item[] Question: Does the paper specify all the training and test details (e.g., data splits, hyperparameters, how they were chosen, type of optimizer, etc.) necessary to understand the results?
    \item[] Answer: \answerNA{} %
    \item[] Justification: The paper is theoretical in nature and does not include experiments.

\item {\bf Experiment Statistical Significance}
    \item[] Question: Does the paper report error bars suitably and correctly defined or other appropriate information about the statistical significance of the experiments?
    \item[] Answer: \answerNA{} %
    \item[] Justification: The paper is theoretical in nature and does not include experiments.

\item {\bf Experiments Compute Resources}
    \item[] Question: For each experiment, does the paper provide sufficient information on the computer resources (type of compute workers, memory, time of execution) needed to reproduce the experiments?
    \item[] Answer: \answerNA{} %
    \item[] Justification: The paper is theoretical in nature and does not include experiments.
    
\item {\bf Code Of Ethics}
    \item[] Question: Does the research conducted in the paper conform, in every respect, with the NeurIPS Code of Ethics \url{https://neurips.cc/public/EthicsGuidelines}?
    \item[] Answer: \answerYes{} %
    \item[] Justification: Our research conforms in every respect with the NeurIPS Code of Ethics.

\item {\bf Broader Impacts}
    \item[] Question: Does the paper discuss both potential positive societal impacts and negative societal impacts of the work performed?
    \item[] Answer: \answerNA{} %
    \item[] Justification: The work is theoretical and we do not see any major or immediate implications
    
\item {\bf Safeguards}
    \item[] Question: Does the paper describe safeguards that have been put in place for responsible release of data or models that have a high risk for misuse (e.g., pretrained language models, image generators, or scraped datasets)?
    \item[] Answer: \answerNA{} %
    \item[] Justification: The work is theoretical.

\item {\bf Licenses for existing assets}
    \item[] Question: Are the creators or original owners of assets (e.g., code, data, models), used in the paper, properly credited and are the license and terms of use explicitly mentioned and properly respected?
    \item[] Answer: \answerNA{} %
    \item[] Justification: This paper does not use existing assets.

\item {\bf New Assets}
    \item[] Question: Are new assets introduced in the paper well documented and is the documentation provided alongside the assets?
    \item[] Answer: \answerNA{} %
    \item[] Justification: This paper does not release new assets.

\item {\bf Crowdsourcing and Research with Human Subjects}
    \item[] Question: For crowdsourcing experiments and research with human subjects, does the paper include the full text of instructions given to participants and screenshots, if applicable, as well as details about compensation (if any)? 
    \item[] Answer: \answerNA{} %
    \item[] Justification: This paper does not involve crowdsourcing nor research with human subjects.

\item {\bf Institutional Review Board (IRB) Approvals or Equivalent for Research with Human Subjects}
    \item[] Question: Does the paper describe potential risks incurred by study participants, whether such risks were disclosed to the subjects, and whether Institutional Review Board (IRB) approvals (or an equivalent approval/review based on the requirements of your country or institution) were obtained?
    \item[] Answer: \answerNA{} %
    \item[] Justification: This paper does not involve crowdsourcing nor research with human subjects.

\end{enumerate}
\fi
\end{document}